\documentclass{article}

\usepackage{microtype}
\usepackage{graphicx}
\usepackage{subfigure}
\usepackage{booktabs} %

\usepackage{hyperref}
\usepackage[utf8]{inputenc} %
\usepackage[T1]{fontenc}    %

\usepackage{url}            %
\usepackage{booktabs}       %
\usepackage{amsfonts}       %
\usepackage{nicefrac}       %
\usepackage{microtype}      %
\usepackage{xcolor}         %
\usepackage{bm}
\usepackage{bbm}
\usepackage{amsthm}
\usepackage{psfrag,amssymb,thmtools}
\usepackage{mathrsfs}
\usepackage{multirow}
\usepackage{xspace}
\usepackage{url}
\usepackage{dsfont}
\usepackage{natbib}
\usepackage{graphicx,graphics}
\usepackage{subfigure}
\usepackage{wrapfig}
\usepackage{xcolor}
\usepackage{enumitem}
\usepackage{booktabs} 
\usepackage{nicefrac} 
\usepackage{yfonts}
\usepackage[euler]{textgreek}
\usepackage[normalem]{ulem}
\usepackage[most]{tcolorbox}
\usepackage{textgreek}

\usepackage[algo2e]{algorithm2e}%
\newcounter{counter}[section]

\newtheorem{thm}{Theorem}[section]
\newtheorem{lem}{Lemma}[section]

\newtheorem{defn}{Definition}

\usepackage[accepted]{icml2024}

\usepackage{amsmath}
\usepackage{amssymb}
\usepackage{mathtools}
\usepackage{amsthm}

\usepackage[capitalize,noabbrev]{cleveref}

\theoremstyle{plain}

\theoremstyle{definition}

\theoremstyle{remark}

\usepackage[textsize=tiny]{todonotes}

\newcommand{\cA}{\mathcal{A}}

\newcommand{\cO}{\mathcal{O}}
\newcommand{\cN}{\mathcal{N}}
\newcommand{\cF}{\mathcal{F}}
\newcommand{\cH}{\mathcal{H}}
\newcommand{\cK}{\mathcal{K}}
\newcommand{\cI}{\mathcal{I}}
\newcommand{\cC}{\mathcal{C}}

\newcommand{\bR}{\mathbb{R}}
\newcommand{\bN}{\mathbb{N}}
\newcommand{\bE}{\mathbb{E}}
\newcommand{\bP}{\mathbb{P}}

\newcommand{\wtU}{\widetilde{U}}
\newcommand{\wtW}{\widetilde{W}}
\newcommand{\wtH}{\widetilde{H}}
\newcommand{\wtL}{\widetilde{L}}
\newcommand{\wtbeta}{\widetilde{\beta}}
\newcommand{\wtcF}{\widetilde{\cF}}
\newcommand{\wtcA}{\widetilde{\cA}}
\newcommand{\wbcup}{\mathop{\bigcup}\displaylimits}
\newcommand{\wbcap}{\mathop{\bigcap}\displaylimits}

\newcommand{\V}[1]{{\bm{\mathbf{\MakeLowercase{#1}}}}} %
\newcommand{\M}[1]{{\bm{\mathbf{\MakeUppercase{#1}}}}} %
\newcommand{\norm}[1]{\left\lVert#1\right\rVert}
\newcommand{\argminE}{\mathop{\mathrm{argmin}}} 

\newcommand{\mmts}{\text{MMTS}}
\newcommand{\cmcp}{\text{CMCPR}}
\newcommand{\mw}{\textit{m}-\text{type worker}}
\newcommand{\app}[1]{\text{applicant}}
\newcommand{\com}[1]{\text{company}}

\icmltitlerunning{Two-sided Competing Matching Recommendation Markets With Quota and Complementary Preferences Constraints}

\begin{document}

\twocolumn[
\icmltitle{Two-sided Competing Matching Recommendation Markets With Quota and Complementary Preferences Constraints}

\begin{icmlauthorlist}
\icmlauthor{Yuantong Li}{ucla}
\icmlauthor{Guang Cheng}{ucla}
\icmlauthor{Xiaowu Dai}{ucla,bio}
\end{icmlauthorlist}

\icmlaffiliation{ucla}{Department of Statistics and Data Science, UCLA;}
\icmlaffiliation{bio}{Department of Biostatistics, UCLA}

\icmlcorrespondingauthor{Yuantong Li}{yuantongli@ucla.edu}
\icmlcorrespondingauthor{Xiaowu Dai}{dai@stat.ucla.edu}

\icmlkeywords{Machine Learning, ICML}

\vskip 0.3in
]

\printAffiliationsAndNotice{}  %

\begin{abstract}
In this paper, we propose a new recommendation algorithm for addressing the problem of two-sided online matching markets with complementary preferences and quota constraints, where agents' preferences are unknown a priori and must be learned from data. The presence of mixed quota and complementary preferences constraints can lead to instability in the matching process, making this problem challenging to solve. To overcome this challenge, we formulate the problem as a bandit learning framework and propose the Multi-agent Multi-type Thompson Sampling ($\mmts$) algorithm. The algorithm combines the strengths of Thompson Sampling for exploration with a new double matching technique to provide a stable matching outcome. Our theoretical analysis demonstrates the effectiveness of $\mmts$ as it can achieve stability and has a total $\widetilde{\mathcal{O}}(Q{\sqrt{K_{\max}T}})$-Bayesian regret with high probability, which exhibits linearity with respect to the total firm's quota $Q$,  the square root of the maximum size of available type workers $\sqrt{K_{\max}}$ and time horizon $T$. In addition, simulation studies also demonstrate $\mmts$' effectiveness in various settings. We provide code
used in our experiments \url{https://github.com/Likelyt/Double-Matching}.
\end{abstract}

\section{Introduction}
Two-sided matching markets have been a mainstay of theoretical research and real-world applications for several decades since the seminal work by \citet{gale1962college}. Matching markets are used to allocate indivisible “goods” to
multiple decision-making agents based on mutual compatibility as assessed via sets of preferences. 
We consider the setting of matching markets with recommender systems, where
preferences are usually unknown in the recommendation process due to the large volume of participants.
One of the key concepts that contribute to the success of matching markets is \emph{stability}, which criterion ensures that all participants have no incentive to block a prescribed matching \citep{roth1982economics}.
Matching markets often consist of participants with \emph{complementary} preferences that can lead to instability \citep{che2019stable}. 
Examples of complementary preferences in matching markets include: firms seeking workers with skills that complement their existing workforce, sports teams forming teams with players that have complementary roles, and colleges admitting students with diverse backgrounds and demographics that complement each other.
Studying the stability issue in the context of complementary preferences is crucial in ensuring the successful functioning of matching markets with complementarities.

In this paper, we propose a novel algorithm and present an in-depth analysis of the problem of complementary preferences in matching markets. 
Specifically, we focus on a many-to-one matching scenario and use the job market as an example.
In our proposed model, there is a set of agents (e.g., firms), each with a limited quota, and a set of arms (e.g., workers), each of which can be matched to at most one
agent. Each arm belongs to a unique type, and each agent wants to match with a minimum quota of arms for each type and a maximum quota of arms from all types. This leads to complementarities in agents' preferences.  Additionally, the agents' preference of arms from each type is unknown a priori and must be learned from data, which we refer to as the \emph{competing matching under complementary preference recommendation problem} ($\cmcp$). 

The main contributions can be summarized as follows.
Our first result is the formulation of $\cmcp$ into a bandit learning framework as described in \citet{lattimore2020bandit}.  Using this framework, we propose a new algorithm, the Multi-agent Multi-type Thompson Sampling ($\mmts$), to solve $\cmcp$. Our algorithm builds on the strengths of Thompson Sampling (TS) \citep{thompson1933likelihood, agrawal2012analysis, russo2018tutorial} in terms of exploration and further enhances it by incorporating a new \emph{double matching} technique to find a stable solution for $\cmcp$, shown in Section \ref{sec: insufficient exploration-ucb-vs-ts}. Unlike the upper confidence bound (UCB) algorithm, the TS method can achieve sufficient exploration by incorporating a deterministic, non-negative bias inversely proportional to the number of matches into the observed empirical means. Furthermore,  the introduced double matching technique uses two stages of matching to satisfy both the type quota and total quota requirements. These two stages' matching mainly consists of using the deferred-acceptance (DA) algorithm from \citet{gale1962college}.

Secondly, we provide the theoretical analysis of the proposed $\mmts$  algorithm. Our analysis shows that $\mmts$ achieves stability and enjoys incentive compatibility (IC). The proof of stability is obtained through a two-stage design of the double matching technique, and the proof of incentive compatibility is obtained through the regret lower bound.
To the best of our knowledge, $\mmts$ is the first algorithm to achieve stability and incentive compatibility in the $\cmcp$.

Finally, our theoretical results indicate that $\mmts$ achieves a Bayesian regret that scales $\widetilde{\mathcal{O}}({\sqrt{T}})$ and is near linear in terms of the total quota of all firms ($Q$). Besides, we find that the Bayesian regret only depends on the square root of the \emph{maximum} number of workers ($K_{\max}$) in one type rather than the square root of the total number of workers ($\sum_{m}K_{m}$) in all types, which is important for the large market. This is a more challenging setting than that considered in previous works such as \citet{liu2020competing} and \citet{jagadeesan2021learning}, which only considers a single type of worker and a quota of one for each firm. 
To address these challenges, we use the eluder dimension \citep{russo2013eluder} to measure the uncertainty set widths and bound the instantaneous regret for each firm, and use the concentration results to measure the probability of \emph{bad events} occurring to get the final regret. Bounding the uncertainty set width is the key step for deriving the regret upper bound of $\mmts$.

The rest of this paper is organized as follows.
Section \ref{sec: problem} introduces basic concepts of $\cmcp$. Section \ref{sec: challenge and sol} presents the challenges of this problem. 
Section \ref{sec: algorithm}  provides  $\mmts$ algorithm, its comparison with UCB-family algorithms, and shows the incapable exploration of the UCB algorithm in CMCPR.
Section \ref{sec: theorem} provides the stability, regret upper bound, and the incentive-compatibility of $\mmts$.
Section \ref{sec: experiments} shows the application of $\mmts$ in simulations, including the distribution of learning parameters, and demonstrates the robustness of $\mmts$ in large markets. Finally, Section \ref{sec: related work} discusses related works.

\begin{figure}
    \centering
    \includegraphics[scale=0.165]{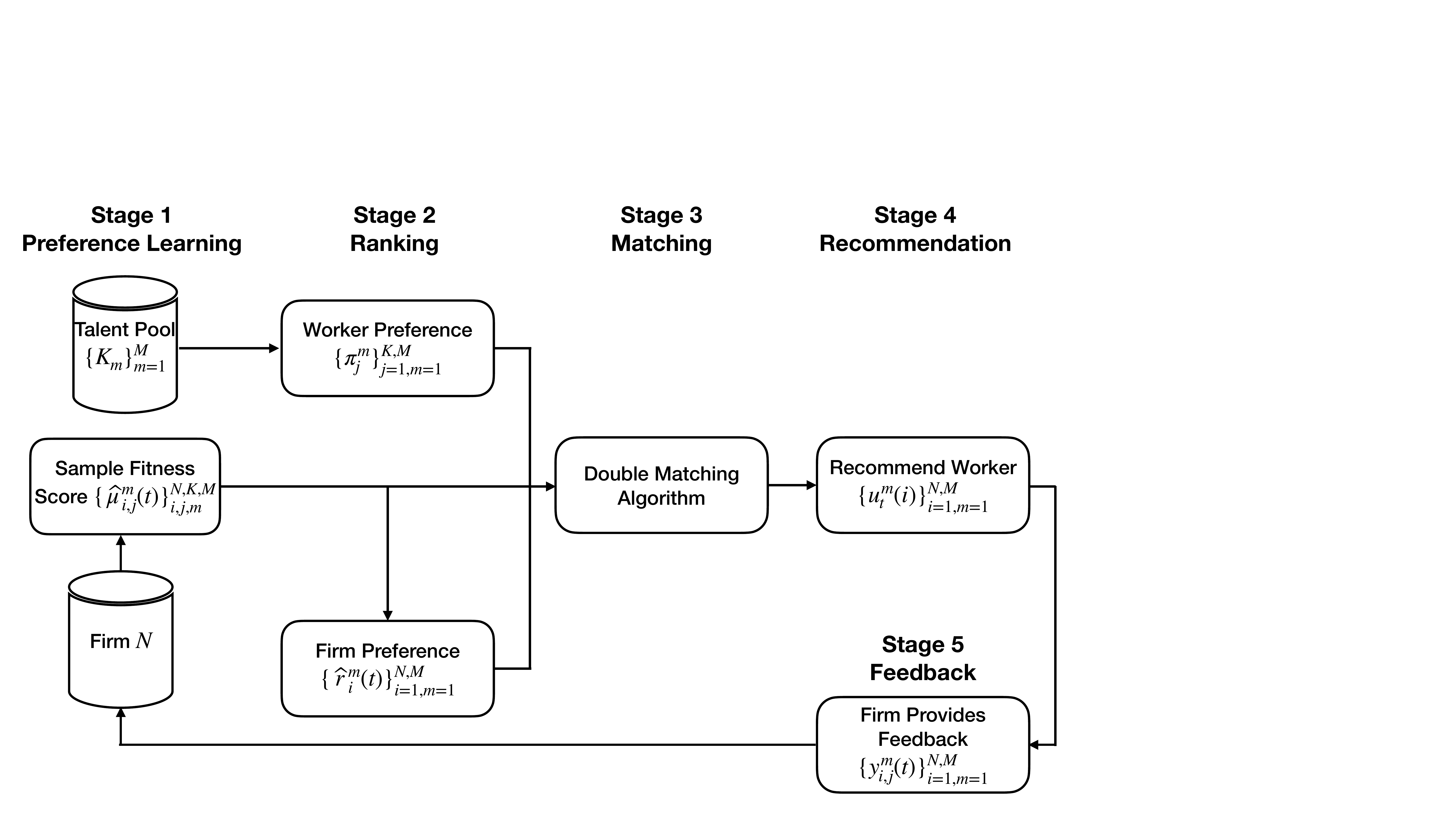}
    \caption{$\mmts$ Algorithm for $\cmcp$ with its application in the job market, including five stages: \textit{preference learning}, \textit{ranking construction}, \textit{matching}, \textit{recommendation}, \textit{feedback collection}.}
    \label{fig:mmts-algo}
\end{figure}

\section{Problem}
\label{sec: problem}
We now describe the problem formulation of the \textbf{C}ompeting \textbf{M}atching under \textbf{C}omplementary \textbf{P}references \textbf{R}ecommendation problem (CMCPR).

\noindent \textbf{Notations.} We define $T$ as the time horizon and assume it is known in advance\footnote{The unknown $T$ can be handled with the doubling trick \citep{auer1995gambling}.}. We denote $[N] = [1,2,..., N]$ where $N\in \mathbb{N}^{+}$.
Define the bold $\V{x} \in \bR^{d}$ be a $d$-dimensional vector.  

\subsection{Environment}
\label{sec: problem formulation}
The matching of workers and firms will be our running example throughout the paper. The organizer is the centralized platform, and the overall goal of the platform is to recommend the best-fit worker and match two-sided participants with their ideal objects over time. We first introduce seven elements in $\cmcp$.

\noindent
\textbf{(I) Participants.} In the centralized platform, there are $N$ firms (agents), denoted by $\mathcal{N} = \{p_1, p_2, ..., p_N\}$, and $M$ types of workers (arms), represented $\mathcal{K}_{m} = \{a_{1}^{m}, a_{2}^{m}, ... a_{K_{m}}^{m}\}, m \in [M]$, where $K_m$ is the number of $m$-th type workers and $M$ is the total types.

\textbf{(II) Quota.} Agent $p_i$ has a minimum quota $q_{i}^{m}$ for $m$-type workers, and a maximum total quota $Q_{i}$ (e.g., seasonal headcount in a company) and we assume $\sum_{i=1}^{M} q_{i}^{m} \leq Q_{i}$. Define the total market quota for all companies as $Q = \sum_{i=1}^{N}Q_{i}$ and the total number of available workers in the market as $K = \sum_{m=1}^{M}K_{m}$. We assume that $Q \ll K$ and $T$ is relatively large.

\noindent
\textbf{(III) Two-sided Complementary Preferences.} 
There are two kinds of preferences: workers to firms' preferences and firms to workers' preferences.

\emph{a. Preferences of $m$-type workers towards firms $\V{\pi}^{m}:\cK_{m} \mapsto \mathcal{N}$}. We assume that there exist fixed preferences from workers to firms, and these preferences are known for the platform. For instance, workers submit their preferences for different firms on the platform.
$\pi_{j,i}^{m}$ represents the rank for $p_{i}$ from the view of $a_{j}^{m}$, and we assume that there are no ties in the rank orders, %
$\V{\pi}_{j}^{m} \subseteq \{\pi_{j,1}^{m}, ..., \pi_{j,N}^{m}\}$.
In other words, $\V{\pi}_{j}^{m}$ is a subset of the permutation of $[N]$. And $\pi_{j,i}^{m} < \pi_{j,i'}^{m}$ implies that $\mw$ $a_{j}^{m}$ prefers firm $p_i$ over firm $p_{i'}$ and as a shorthand, denoted as $p_{i} <_{j}^{m} p_{i'}$. %
This known worker-to-firm preference is a mild and common assumption in matching market literature \citep{liu2020competing,liu2021bandit,li2022rate}.

\emph{b. Preferences of firms towards $\mw$s $\V{r}^{m}:\mathcal{N}\mapsto \cK_{m}$}. The preferences of firms towards workers are fixed, but \emph{unknown}. The goal of the platform is to infer these unknown preferences through historical matching data. We denote $r_{i,j}^{m}$ as the true rank of worker $a_{j}^{m}$ in the preference list of firm $p_{i}$, and assume there are no ties. $p_{i}$'s preferences towards workers is represented by $\V{r}_{i}^{m}$, which is a subset of the permutation of $[\cK_{m}]$. $r_{i,j}^{m} < r_{i, j'}^{m}$ implies that firm $p_{i}$ prefers worker $a_{j}^{m}$ over worker $a_{j'}^{m}$.

\subsection{Policy}
\noindent
\textbf{(IV) Matching Policy.}
$u^{m}_{t}(p_{i}): \cN \mapsto \cK_{m}$ is a recommendation mapping function from $p_{i}$ to  $\mw$s at time $t$.

\noindent
\textbf{(V) Stable Matching and Optimal Matching.} 
\label{def: stable matching}
We introduce key concepts in matching fields \citep{roth2008deferred}.

\begin{defn}[Blocking pair] 
A matching $u$ is blocked by firm $p_{i}$ if $p_{i}$ prefers being single to being matched with $u(p_{i})$, i.e. $p_{i} >_{i} u(p_{i})$. A matching $u$ is blocked by a pair of firm and worker $(p_{i}, a_{j})$ if they each prefer each other to the partner they receive at $u$, i.e. $a_{j} >_{i} u(p_{i})$ and $p_{i} >_{j} u^{-1}(a_{j}).$
\end{defn}

\begin{defn}[Stable Matching] 
A matching $u$ is stable if it isn't blocked by any individual or pair of workers and firms.
\end{defn}

\begin{defn}[Valid Match]
With true preferences from both sides, arm $a_{j}$ is called a \textit{valid match} of agent $p_{i}$ if there exists a stable matching according to those rankings such that $a_{i}$ and $p_{j}$ are matched.
\end{defn}

\begin{defn}[Agent Optimal Match]
Arm $a_{j}$ is an \textit{optimal match} of agent $p_{i}$ if it is the most preferred valid match.  
\end{defn}

Given two-sided true preferences, the deferred-acceptance (DA) algorithm \citep{gale1962college} will provide a stable matching outcome. The matching result by the DA algorithm is always optimal for members of the proposing side, and we denote the agent-optimal policy as $\{\overline{u}_{i}^{m}\}_{m=1}^{M}$.

In $\cmcp$, it is worth mentioning that each firm has a minimum quota constraint $\V{q}_{i} = [q_{i}^{1}, ..., q_{i}^{M}]$ for all type workers to fill and total quota cap is $Q_{i}$. Therefore, we define the concept of stability as the absence of ``blocking pairs" across all types of workers and firms.\footnote{The discussion of the feasibility of the stable matching in $\cmcp$ is in Appendix \ref{supp: fea of stable matching}.} %

\noindent
\textbf{(VI) Matching Score.}
If $p_i$ is matched with  $a_j^m$ at time $t$, $p_i$ provides a noisy reward $y_{i,j}^m(t)$ which is assumed to be the \emph{true matching score} $\mu_{i,j}^m$ plus a noise term $\epsilon_{i,j}^m(t)$, 
\begin{equation}
\label{eq: utility model}
    \begin{aligned}
        y_{i,j}^{m}(t) = \mu_{i,j}^{m}+ \epsilon_{i,j}^{m}(t),
    \end{aligned}
\end{equation}
$\forall i, j, m, t\in [N], [K_{m}], [M],[T]$, where we assume that $\epsilon_{i,j}^m(t)$'s are independently drawn from a sub-Gaussian random variable with parameter $\sigma$. That is, for every $\alpha \in \bR$, it is satisfied that 
$\mathbb{E}[\exp(\alpha \epsilon_{i,j}^{m}(t))] \leq \exp(\alpha^2\sigma^2/2)$.

\noindent
\textbf{(VII) Regret.} 
Based on model (\ref{eq: utility model}), we denote the matching score for $p_{i}$ as $\V{y}_{i}^{m}(t):= \V{y}_{i,u_{t}^{m}(p_{i})}(t)$ in short.
Define the \emph{firm-optimal regret with $m$-type worker} for $p_{i}$  as 
\begin{equation}
    \begin{aligned}
        R_{i}^{m}(T, \theta):= \sum_{t=1}^{T}[\mu_{i,\overline{u}_{i}^{m}} - \mu_{i,u_{t}^{m}(p_{i})}(t)|\
        \theta],
    \end{aligned}
\end{equation}
where denote $\theta$ as the sampled problem instance from the distribution $\Theta$.
$R_{i}^{m}(T, \theta)$ represents
the total expected score difference between the policy  $u_{i}^{m} := \{u_{t}^{m}(p_{i})\}_{t=1}^{T}$ and the optimal policy $\overline{u}_{i}^{m}$ in hindsight.

As each firm have to recruit $M$ types workers with total quota $Q_{i}$, the \emph{total firm-optimal stable regret} for $p_{i}$ is defined as 
\begin{equation}
R_{i}(T,\theta):= \sum_{m=1}^{M} R_{i}^{m}(T, \theta).
\end{equation}
Finally, define the \emph{Bayesian social welfare gap} (BSWG) $\mathfrak{R}(T)$ as the expected regret over all firms and problem instance, 
\begin{equation}
    \mathfrak{R}(T):= \mathbb{E}_{\theta \in \Theta}\left[\sum_{i=1}^{N}R_{i}(T,\theta)\right].
\end{equation}
The goal of the centralized platform is to design a learning algorithm that achieves stable matchings through learning the firms' complementary preferences for multiple types of workers preciously from the previous matchings for a better recommendation.
This is equivalent to designing an algorithm that minimizes BSWG $\mathfrak{R}(T)$.

\section{Challenges and Solutions}
\label{sec: challenge and sol}
When preferences are unknown a priori in matching markets, the stability issue while satisfying complementary preferences and quota requirements is a challenging problem due to the interplay of multiple factors. %

\textbf{Challenge 1: How to design a stable matching algorithm to solve complementary preferences?} 
This is a prevalent issue in real-world applications such as hiring workers with complementary skills in hospitals and high-tech firms or admitting students with diverse backgrounds in college admissions. Despite its importance, no implementable algorithm is currently available to solve this challenge. In this paper, we propose a novel approach to resolving this issue by utilizing a novel designed \emph{double matching} (Algorithm  \ref{algo:db-matching}) to marginalize complementary preferences and achieve stability. Our algorithm can efficiently learn a stable matching result using historical matching data, providing a practical solution to $\cmcp$.

\textbf{Challenge 2: How to balance exploration and exploitation to achieve the sublinear regret?} 
The platform must find a way to recommend the most suitable workers to firms to establish credibility among workers and firms to stay at the platform towards achieving optimal matching. Compared to traditional matching algorithms, the CMCPR is not a one-time recommendation algorithm but a \textit{recycled} online recommendation matching algorithm with supply and demand consideration (workers and firms), which is more challenging as it requires more time to balance this trade-off. In addition, the classic UCB bandit methods could not function well in exploration and suffer sublinear regret demonstrated in Section \ref{sec: insufficient exploration-ucb-vs-ts}. To overcome this challenge, we propose the use of a sampling algorithm, which allows for better exploration and achieves sublinear regret. %

\section{Algorithms}
\label{sec: algorithm}
In this section, we propose the  Multi-agent Multi-type Thompson Sampling algorithm  ($\mmts$), which aims to learn the true preferences of all firms over all types of workers, achieve stable matchings, and minimize firms' Bayesian regret. We provide a description of $\mmts$ and demonstrate the benefits of using the sampling method. The overall $\mmts$ algorithm procedure is in Figure \ref{fig:mmts-algo}. The computational complexity of $\mmts$ is in Appendix \ref{sec: com complexity}.

\begin{algorithm}[t]
\SetAlgoLined
	\DontPrintSemicolon
	\SetAlgoLined
	\SetKwInOut{Input}{Input}
    \SetKwInOut{Output}{Output}
	\Input{Time horizon $T$; firms' priors $(\V{\alpha}^{m,0}_{i}, \V{\beta}^{m,0}_{i}), \forall i, m \in [N], [M]$; workers' preference $\V{\pi}^{m}, \forall m \in [M]$.}
    \For{$t \in \{1, ..., T\}$}{
    	\textbf{\textsc{Step 1: preference learning Stage}} \\
        \hspace{0.5cm} Sample estimated mean reward $\widehat{\V{\mu}}_{i}^{m}(t)$ over all types of workers (Algo. \ref{algo:ts-sampling})\\
        \textbf{\textsc{Step 2: Ranking Construction Stage}} \\ Construct all firms' estimated rankings $\{\widehat{\V{r}}^{m}_{i}(t)\}_{i=1, m=1}^{N, M}$ according $\widehat{\V{\mu}}_{i}^{m}(t)$.\\
    	\textbf{\textsc{Step 3: Double Matching Stage}}\\
        \hspace{0.5cm}  Get the matching result $\V{u}_{t}^{m}(p_{i}), \forall i \in [N], m \in [M]$ from the \emph{double matching} in Algo \ref{algo:db-matching}.\\
        \textbf{\textsc{Step 4: Recommending and Collecting Feedback Stage}}\\
        \hspace{0.5cm} Each firm receives its corresponding rewards from recommended all types of workers $\V{y}_{i}^{m}(t)$.\\
        \textbf{\textsc{Step 5: Updating Belief Stage}}\\
        \hspace{0.5cm} Based on received rewards, the platform updates firms' posterior belief.\\
    }
	\caption{Multi-agent Multi-type Thompson Sampling Algorithm ($\mmts$)}
 \label{algo:mm-ts}
\end{algorithm}

\subsection{Algorithm Description}
\label{sec: algorithm desc}
The $\mmts$ (Algorithm  \ref{algo:mm-ts}) is composed of five stages, \emph{preference learning stage}, \emph{ranking construction stage}, \emph{double matching stage}, \emph{collecting feedback stage}, and \emph{updating belief stage}.
At each matching step $t$, $\mmts$ iterates these five steps.

\noindent
\textbf{Step 1: Preference Learning Stage.} (Algorithm \ref{algo:ts-sampling}). For agent $p_{i}$, platform samples the mean feedback (reward) $\widehat{\mu}_{i,j}(t)$ of arm $a_{j}^{m}$ from distribution $\mathcal{P}^{m}_{j}$ with estimated  parameters $(\alpha^{m,t-1}_{i,j}, \beta^{m,t-1}_{i,j})$ from the historical matching data. 

\textbf{Step 2: Ranking Construction Stage.}
Then the platform sorts these workers within each type according $\{\widehat{\mu}_{i,j}(t)\}$ in descending order and gets the estimated rank $\widehat{\V{r}}^{m}(t) = \{\widehat{\V{r}}^{m}_{i}(t)\}_{i=1,m=1}^{N,M}$ where we denote  $\widehat{\V{r}}^{m}_{i}(t) = \{\widehat{\V{r}}_{i,j}^{m}(t)\}_{j=1}^{K_{m}}$.

\noindent
\textbf{Step 3: Double Matching Stage.} (Algorithm \ref{algo:db-matching}). 
With sampled mean reward $\V{\widehat{\mu}}(t) := \{\widehat{\mu}_{i,j}^{m}(t)\}_{i=1,j=1,m=1}^{N,K_{m}, M}$, estimated ranks $\{\widehat{\V{r}}^{m}(t)\}_{m=1}^{M}$, quota constraints $\{Q_{i}\}_{i=1}^{N}$, the double matching algorithm provides the final matching result with two-stage matchings.

The goal of the first match is to allow all firms to satisfy their minimum type-specific quota ${q_{i}^{m}}$ first followed by sanitizing the status quo as a priori. The second match is to fill the left-over positions $\widetilde{Q}_{i}$ (defined below) for each firm and match firms and workers without type consideration. 

\paragraph{a). First Match:} The platform implements the type-specific DA (Algo. \ref{algo:da-type} in Appendix) given quota constraints $\{q_{i}^{m}\}_{i=1, m=1}^{N,M}$. The matching road map starts from matching all firms with type from 1 to $M$ and returns the matching result $\{\widetilde{u}_{t}^{m}(p_{i})\}_{m\in [M]}$. This step can be implemented in parallel.

\paragraph{b). Sanitize Quota:} After the first match, the centralized platform sanitizes each firm's left-over quota $\widetilde{Q}_{i} = Q_{i} - \sum_{m=1}^{M}q_{i}^{m}$. If there exists a firm $p_{i}, s.t., \widetilde{Q}_{i} > 0$, then the platform will step into the second match. For those firms like $p_{i}$ whose leftover quota is zero $\widetilde{Q}_{i} = 0$, they and their matched workers will skip the second match.

\paragraph{c). Second Match:} When rest firms and workers continue to join in the second match, the centralized platform implements the standard DA in Algorithm \ref{algo:da-no-type} without type consideration. 
That is, the platform re-ranks the rest $M$ types of workers who do not have a match in the first match for firms, and fill available vacant positions.
It is worth noting that in Algorithm \ref{algo:da-no-type}, each firm will not propose to the previous workers who rejected him/her already or matched in Step 1.
Then firm $p_{i}$ gets the corresponding matched workers $\breve{u}_{t}(p_{i})$ in the second match. 
Finally, the platform merges the first and second results to obtain a final matching
$\V{u}_{t}^{m}(p_{i}) = \text{Merge}(\widetilde{u}_{t}^{m}(p_{i}), \breve{u}_{t}(p_{i})), \forall i, m \in [N], [M]$.

\noindent
\textbf{Step 4: Recommending and Collecting Feedback Stage.} When the platform broadcasts the matching result $\V{u}_{t}^{m}(p_{i})$ to all firms, each firm then receives its corresponding stochastic reward $\V{y}_{i}^{m}(t), \forall i \in [N], m \in [M]$. 

\noindent
\textbf{Step 5: Updating Belief Stage.}
After receiving these noisy rewards, the platform updates firms' belief (posterior) parameters as follows: $(\V{\alpha}^{m,t}_{i}, \V{\beta}^{m,t}_{i}) = \text{Update}(\V{\alpha}^{m,t-1}_{i}, \V{\beta}^{m,t-1}_{i},\V{y}_{i}^{m}(t)), \forall i \in [N], \forall m \in [M].$

\begin{algorithm}[t]

\SetAlgoLined
	\DontPrintSemicolon
	\SetAlgoLined
	\SetKwInOut{Input}{Input}
    \SetKwInOut{Output}{Output}
	\Input{Time horizon $T$; firms' priors $(\V{\alpha}^{m,0}_{i}, \V{\beta}^{m,0}_{i}), \forall i \in [N], \forall m \in [M]$.}
    \textbf{Sample}: 
    Sample mean reward
    $\widehat{\mu}_{i,j}^{m}(t) \sim \mathcal{P}(\alpha^{m,t-1}_{i,j}, \beta^{m,t-1}_{i,j})$, $\forall i, m, j \in [N],[M],[\cK_{m}]$.\\    
	\textbf{Sort}: Sort estimated mean feedback $\widehat{\mu}_{i,j}^{m}(t)$ in descending order and get the estimated rank $\widehat{\V{r}}^{m}_{i}(t)$.\\
    \textbf{Output}: The estimated rank $\widehat{\V{r}}^{m}_{i}(t)$ and  the estimated mean feedback $\widehat{\V{\mu}}_{i}^{m}(t)$, $\forall i,m \in [N], [M]$.\\
	\caption{Preference Learning Stage}
 \label{algo:ts-sampling}
\end{algorithm}
\begin{algorithm}[t]
\SetAlgoLined
	\DontPrintSemicolon
	\SetAlgoLined
	\SetKwInOut{Input}{Input}
    \SetKwInOut{Output}{Output}
	\Input{Estimated rank $\widehat{\V{r}}(t)$, estimated mean $\widehat{\V{\mu}}_{i}^{m}(t)$, type quota $q_{i}^{m}, \forall m \in [M], i \in [N]$ and total quota $Q_{i}, \forall i \in [N]$; workers' preference $\{\V{\pi}^{m}\}_{m \in [M]}$.}
	\textbf{\textsc{Step 1: First Match}}  \\
    \hspace{0.5cm} Given estimated ranks $\widehat{\V{r}}(t)$ and all workers' preferences $\V{\pi}^{m}$, the platform operates the firm-propose DA Algo and return the matching $\{\widetilde{u}_{t}^{m}(p_{i})\}_{i=1, m}^{N, M}$.\\
    \textbf{\textsc{Step 2: Sanitize Quota}} \\
    \hspace{0.5cm} Sanitize whether all firms' positions have been filled. For each company $p_{i}$, if $Q_{i} - \sum_{m=1}^{M}q_{i}^{m} > 0$, set the left quota as $\widetilde{Q}_{i} \gets Q_{i} - \sum_{m=1}^{M}q_{i}^{m}$ for firm $p_{i}$.\\ 
    \textbf{\textsc{Step 3: Second Match}}\\
    \uIf{$\widetilde{\M{Q}} \neq 0$}{
        Given left quota $\{\widetilde{Q}_{i}\}_{i \in [N]}$, estimated means $\widehat{\V{\mu}}(t)$, and workers' preferences $\{\V{\pi}^{m}\}_{m \in [M]}$, the platform runs the firm-propose DA and return the matching $\breve{u}_{t}(p_{i})$.\;
    }\uElse{
        Set the matching $\breve{u}_{t}(p_{i}) = \emptyset$.\\
    }
    \Output{The matching $u_{t}^{m}(p_{i}) \gets \text{Merge}(\widetilde{u}_{t}^{m}(p_{i}), \breve{u}_{t}(p_{i}))$ for all firms.}
	\caption{Double Matching}
 \label{algo:db-matching}
\end{algorithm}

\subsection{Incapable Exploration}
\label{sec: insufficient exploration-ucb-vs-ts}
\begin{figure}
    \centering
    \includegraphics[width=0.47\textwidth]{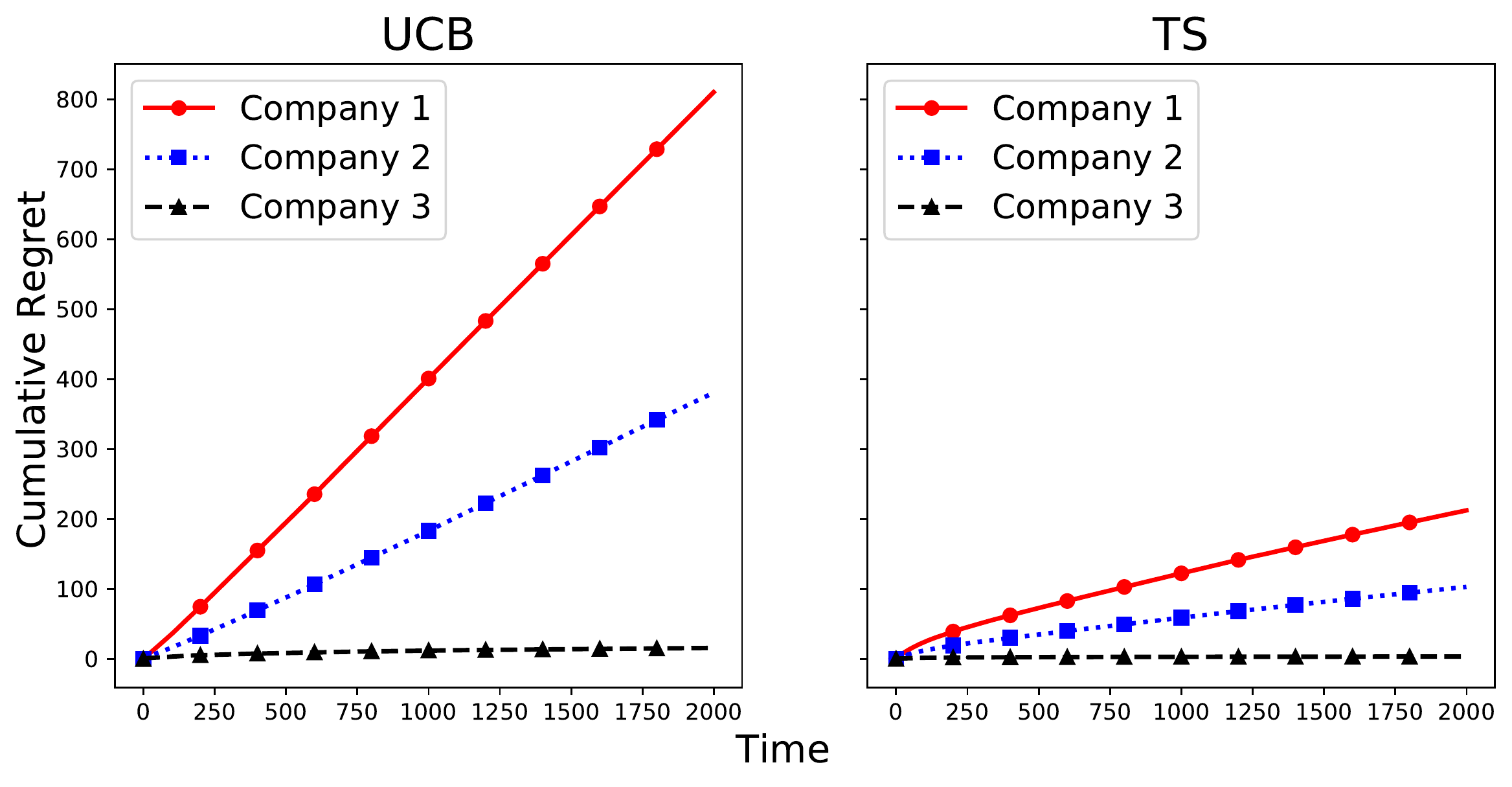}
    \caption{A comparison of centralized UCB and TS that demonstrates the incapable exploration of UCB.}
    \label{fig:ts_vs_ucb}
\end{figure}
We show why the sampling method has an advantage over the UCB method in estimating worker ranks.
We find that centralized UCB suffers linear firm-optimal stable regret in some cases and show it in Appendix \ref{supp-sec: insufficient exploration-ucb-vs-ts} with detailed experimental setting and analysis. %

\textit{Why sampling method is capable of avoiding the curse of linear regret?} By the property of sampling shown in Algorithm \ref{algo:ts-sampling}. Firm $p_{i}$'s initial prior over worker $a_{i}$ is a uniform random variable, and thus $r_{j}(t) > r_{i}(t)$ with probability $\widehat{\mu}_{j} \approx \mu_{j}$, rather than \emph{zero}! This differs from the UCB style method, which cannot update $a_{i}$'s upper bound due to lacking exploration over $a_{i}$. The benefit of TS is that it can occasionally explore different ranking patterns, especially when there exists such a previous example. In Figure \ref{fig:ts_vs_ucb}, we show a quick comparison of centralized UCB \citep{liu2020competing} in the settings shown above and $\mmts$ when $M=1, Q=1, N=3, K=3$. The UCB method produces a linear regret for firm 1 and firm 2. However, the TS method achieves a sublinear regret in firm 1 and firm 2. %

\section{Properties of \mmts: Stability and  
Regret}
\label{sec: theorem}
Section \ref{sec: stability} demonstrates the double matching algorithm can provide the stability property for $\cmcp$. Section \ref{sec: regret} establishes the Bayesian regret upper bound for all firms when they follow the $\mmts$. Section \ref{sec: strategy} discusses the incentive-compatibility property of the $\mmts$.

\subsection{Stability}
\label{sec: stability}

In the following theorem, we show the double matching algorithm (Algo.\ref{algo:db-matching}) provides a stable matching solution in the following theorem.

\begin{thm}
\label{thm: proof stability}
Given two sides' preferences from firms and $M$ types of workers. The double-matching procedure can provide a firm-optimal stable matching solution $\forall t \in [T]$. 
\end{thm}
\vspace{-5mm}
\begin{proof}
The sketch proof of the stability property of $\mmts$ is two steps, naturally following the design of $\mmts$. The first match is conducted in parallel, and the output is stable and guaranteed by \citep{gale1962college}. As the need of $\mmts$, before the second match, firms without leftover quotas ($\tilde{Q} = 0$) will quit the second round of matching, which will not affect the stability. After the quota sanitizing stage, firms and leftover workers will continue to join in the second matching stage, where firms do not need to consider the type of workers designed by double matching. And the DA algorithm still provides a stable result based on each firm's \emph{sub-preference} list. The reason is that for firm $p_{i}$, all previous possible favorite workers have been proposed in the first match. If they are matched in the first match, they quit together, which won't affect the stability property; otherwise, the worker has a better candidate (firm) and has already rejected the firm $p_{i}$. So for each firm $p_{i}$, it only needs to consider a sub-preference list excluding the already matched workers in the first match and the proposed workers in the first match. It will provide a stable match in the second match and won't be affected by the first match. So, the overall double matching is a stable algorithm.
The detailed proof can be found in Appendix Section \ref{supp-thm:  stability of mmts}.
\end{proof}
\vspace{-4mm}
\noindent

\subsection{Bayesian Regret Upper Bound}
\label{sec: regret}
Next, we provide $\mmts$'s Bayesian total firm-optimal regret upper bound. 
\begin{thm}
\label{thm: regret upper bound}
Assume $K_{\max} = \max\{K_{1},..., K_{M}\}, K = \sum_{m=1}^{M}K_{m}$, with probability $1-1/QT$, when all firms follow the $\mmts$ algorithm, firms together will suffer the Bayesian expected regret 
$$\mathfrak{R}(T) \leq 8Q\log(QT) \sqrt{K_{\max}T} + NK/Q.$$
\end{thm}
\vspace{-4mm}
\begin{proof}
     The detailed proof can be found in Appendix \ref{supp-sec: regret upper bound}.
\end{proof}
\vspace{-4mm}
\noindent
\textit{Remark}. The derived Bayesian regret bound, which is dependent on the square root of the time horizon $T$ and a logarithmic term, is nearly rate-optimal. Additionally, we examine the dependence of this regret bound on other key parameters. The first of which is a near-linear dependency on the total quota $Q$. Secondly, the regret bound is dependent only on the \textit{square root} of the maximum worker $K_{\max}$ of one type, as opposed to the total number of workers, $\sum_{m=1}^{M}K_{m}$ in previous literature  \citep{liu2020competing, jagadeesan2021learning}. This highlights the ability of our algorithm, $\mmts$, to effectively capture the interactions of multiple types of matching in $\cmcp$ for the adaptation to the large market ($K$).
The second term in the regret is a constant, which is only dependent on constants $N, K$, and the total quota $Q$. Notably, if we assume that each $q_{i} = 1$ and $Q_{i} = M$, then $NK/Q$ will be reduced to $NK/(NM) = K/M$, which is an unavoidable regret term due to the exploration in bandits \citep{lattimore2020bandit}. This also demonstrates that the Bayesian total cumulative firm-optimal exploration regret is only dependent on the \emph{average} number of workers of each type available in the market, as opposed to the \textit{total} number of workers or the maximum number of workers available of all types. Additionally, if one $Q_{i}$ is dominant over other firms' $Q_{i}$, then the regret will mainly be determined by that dominant quota $Q_{i}$ and $K_{\max}$, highlighting the inter-dependence of this complementary matching problem.

\subsection{Incentive-Compatibility}
\label{sec: strategy}
In this section, we discuss the incentive-compatibility property of $\mmts$. That is if one firm does not match the worker that $\mmts$ (platform) recommended when all other firms follow $\mmts$ recommended matching objects, which is equivalent to that firm submitting ranking preferences different from the sampled ranking list from $\mmts$, and we know that firm cannot benefit (matched with a better worker than his optimal stable matching worker) over a sublinear order. 
As we know, 
\citep{dubins1981machiavelli}
discussed the \emph{Machiavelli} firm could not benefit from incorrectly stating their true preference when there exists a unique stable matching. However, when one side's preferences are unknown and need to be learned through data, this result no longer holds. Thus, the maximum benefits that can be gained by the Machiavelli firm are under-explored in the setting of learning in matching. \citep{liu2020competing} discussed the benefits that can be obtained by Machiavelli firms when other firms follow the centralized-UCB algorithm with the problem setting of one type of worker and quota equal one in the market.

We now show in $\cmcp$,  when all firms except one $p_{i}$ accept their $\mmts$ recommended workers from the 
matching platform, the firm $p_{i}$ has an incentive also to follow the sampling rankings in a \emph{long horizon}, so long as the matching result do not have multiple stable solutions. Now we establish the following lemma, which is an upper bound of the expected number of pulls that a firm $p_{i}$ can match with a $m$-type worker that is better than their optimal $m$-type workers, regardless of what workers they want to match.

Let's use $\cH_{i,l}^{m}$ to define the achievable \emph{sub-matching} set of  $u^{m}$ when all firms follow the $\mmts$, which represents firm $p_{i}$ and $\mw$ $a_{l}^{m}$ is matched such that $a_{l}^{m} \in u_{i}^{m}$. Let \textUpsilon$_{u^{m}}(T)$ be the number of times sub-matching $u^{m}$ is played by time $t$. We also provide the blocking triplet in a matching definition as follows.

\begin{defn}[Blocking triplet]
A blocking triplet $(p_{i}, a_{k}, a_{k'})$ for a matching $u$ is that there must exist a firm $p_{i}$ and worker $a_{j}$ that they both prefer to match with each other than their current match. That is, if $a_{k'} \in u_{i}$, $\mu_{i, k'} < \mu_{i, k}$ and worker $a_{k}$ is either unmatched or $\pi_{k,i} < \pi_{k, u^{-1}(k)}$.  
\end{defn}

The following lemma presents the upper bound of the number of matching times of $p_{i}$ and $a_{l}^{m}$ by time $T$, where $a_{l}^{m}$ is a \emph{super optimal} $\mw$ (preferred than all stable optimal $\mw$s under true preferences), when all firms follow the $\mmts$.

\begin{lem}
\label{lem: IC lemma}
Let \textUpsilon$_{i,l}^{m}(T)$ be the number of times a firm $p_{i}$ matched with a $m$-type worker such that the mean reward of $a_{l}^{m}$ for firm $p_{i}$ is greater than $p_{i}$'s optimal match $\overline{u}_{i}^{m}$, which is $\mu_{i,a_{l}^{m}}^{m} > \underset{a_{j}^{m} \in \overline{u}_{i}^{m}}{\max}\mu_{i, j}^{m}$. Then the expected number of matches between $p_{i}$ and $a_{l}^{m}$ is upper bounded by
\begin{equation*}
\begin{aligned}
&\bE[\text{\textUpsilon}_{i,l}^{m}(T)] \leq 
\underset{S^{m} \in \cC(\cH_{i,l}^{m})}{\min}\\
&\sum_{(p_{j}, a_{k}^{m}, a_{k'}^{m})\in S^{m}}
        \big(
            C^{m}_{i,j,k'}(T) +
            \frac{\log(T)}{d(\mu_{j,\overline{u}_{i,\min}^{m}}, \mu_{j, k'})}
        \big),
\end{aligned}
\end{equation*}
where $\overline{u}_{i,\min}^{m} = \underset{a_{k}^{m} \in \overline{u}_{j}^{m}}{\argminE}\hspace{1mm} \mu_{i,k}^{m}$, and  $C^{m}_{i,j,k'} = \cO((\log(T))^{-1/3})$.
\end{lem}

Then we show the benefit (lower bound of the regret) of Machiavelli firm $p_{i}$ can gain by not following the $\mmts$ recommended workers. Let's define the \emph{super reward gap} as $\overline{\Delta}_{i,l}^{m} = \underset{a_{j}^{m} \in \overline{u}_{i}^{m}}{\max}\mu_{i, j}^{m} - \mu_{i, l}^{m}$, where $a_{l}^{m} \notin \overline{u}_{i}^{m}$.
\begin{thm}
\label{thm: IC}
Suppose all firms other than firm $p_{i}$ follow the preferences according to the $\mmts$ to the centralized platform. Then the following upper bound on firm $p_{i}$'s optimal regret for $m$-type workers holds:
\begin{equation*}
    \begin{aligned}
        &R_{i}^{m}(T,\theta) \geq \sum_{l: \overline{\Delta}_{i,l}^{m} < 0} \overline{\Delta}_{i,l}^{m}
        \Bigg[
        \underset{S^{m} \in \cC(\cH_{i,l}^{m})}{\min}
        \\
        &\sum_{(p_{j}, a_{k}^{m}, a_{k'}^{m})\in S^{m}}
        \bigg(
        C^{m}_{i,j,k'} +
        \frac{\log(T)}{d(\mu_{j, \overline{u}_{i,\min}^{m}}, \mu_{j, k'})}
        \bigg)
        \Bigg].
    \end{aligned}
\end{equation*}
\end{thm}

\noindent
This result can be directly derived from Lemma \ref{lem: IC lemma}.
Theorem \ref{thm: IC} demonstrates that there is no sequence of preferences that a firm can manipulate and does not follow $\mmts$ recommended workers that would achieve negative optimal regret and its absolute value greater than $\cO(\log T)$. Considering M types together for firm $p_{i}$, this magnitude remains $\cO(M\log T)$. Theorem \ref{thm: IC} confirms that when there is a unique stable matching, firms cannot gain a significant advantage in terms of firm-optimal stable regret due to incorrect estimated preferences if others follow $\mmts$. 

An example is provided in Section \ref{exp: negative regret} to illustrate this incentive compatibility property.
Figure \ref{fig: neg regret} illustrates the total regret, with solid lines representing the aggregate regret over all types for each firm and dashed lines representing each type's regret. It is observed that the type I regret of $p_{1}$ is negative, owing to the inaccuracies in the rankings estimated for both $p_{1}$ and $p_{2}$. A detailed analysis of this negative regret pattern is given in Appendix Section \ref{sec: learning parameters}.

\section{Experiments}
\label{sec: experiments}
In this section, we present simulation results to demonstrate the effectiveness of $\mmts$ in learning firms' unknown preferences. The detailed experiment setup and the result can be found in Appendix Section \ref{supp-sec: exps}. Section \ref{exp: negative regret} presents two examples to analyze the underlying causes of the novel phenomenon of negative regret (\textit{gain benefit by matching with over-optimal workers}) and large market effect. Appendix Section \ref{sec: learning parameters} showcases the distribution of learning parameters and provides insight into reasons for non-optimal stable matchings. Additionally, we demonstrate the robustness of $\mmts$ in large markets in Appendix \ref{sec: large market}. All simulation results are run in 100 trials.

\subsection{Two Examples}
\label{exp: negative regret}

\noindent
\textbf{Example 1.} There are $N=2$ firms, $M=2$ types of workers, and there are $K_{m}=5, \forall m \in [M]$. The quota $q_{i}^{m}$ for each type and each firm $p_{i}$ is 2, and the total quota/capacity for each firm is $Q_{i}=5$. The time horizon is $T = 2000$. 

\noindent
\textbf{Preferences.} True preferences from workers to firms and from firms to workers are all randomly generated. 
Preferences from workers to firms' $\{\V{\pi}^{m}\}_{m=1}^{M}$ are fixed and known. We use the data scientist (\emph{D} or \emph{DS}) and software developer engineer (\emph{S} or \emph{SDE}) as our example. The following are true preferences: $D_{1}: p_{1} \succ p_{2},
        D_{2}: p_{1} \succ p_{2}, 
        D_{3}: p_{2} \succ p_{1}, 
        D_{4}: p_{1} \succ p_{2},$
$D_{5}: p_{2} \succ p_{1},
        S_{1}: p_{1} \succ p_{2},
        S_{2}: p_{1} \succ p_{2},
        S_{3}: p_{2} \succ p_{1},$
$S_{4}: p_{2} \succ p_{1},
        S_{5}: p_{1} \succ p_{2},$
        and
\begin{equation*}
\label{eq: example preference 1}
    \begin{aligned}
       &\pi_{1}^{1}: 
          D_{4}\succ D_{2}\succ D_{3}\succ D_{5}\succ D_{1}, \\
        & \pi_{1}^{2}: 
          S_{1}\succ S_{4}\succ S_{5}\succ S_{2}\succ S_{3}, \\
        &\pi_{2}^{1}:
         D_{2}\succ D_{3}\succ D_{1}\succ D_{5}\succ D_{4},\\
         & \pi_{2}^{2}: 
         S_{4}\succ S_{2}\succ S_{5}\succ S_{1}\succ S_{3}.
    \end{aligned}
\end{equation*}
The true matching scores of each worker for firms are sampled from $U([0,1])$ and are available in Appendix Table \ref{table: mean reward of negative regret}.
In addition, feedback $y_{i,j}^{m}(t)$ (0 or 1) provided by firms is generated by $\text{Bernoulli}(\mu_{i,j}^{m}(t))$. 
If two sides' preferences are known, the firm optimal stable matching is $\bar{u}_{1} = \{[D_{2}, D_{4}], [S_{5}, S_{1}, S_{3}]\}$, $\bar{u}_{2} = \{ [D_{3}, D_{1}, D_{5}], [S_{4}, S_{2}]\}$ by the double matching algorithm. However, if firms' preferences are unknown, $\mmts$ can learn these unknown preferences and attain the optimal stable matching while achieving a sublinear regret for each firm. 

\noindent
\textbf{$\mmts$ Parameters.} We set priors $\alpha_{i,j}^{m,0} = \beta_{i,j}^{m,0} = 0.1, \forall i \in [N], \forall j \in [K_{m}], \forall m \in [M]$ to limit the strong impact of the prior belief. The update formula for each firm $p_{i}$ at time $t$ of the $m$-type worker $a_{j}^{m}$: $\alpha_{i,j}^{m,t+1} = \alpha_{i,j}^{m,t} + 1$ if the worker $a_{j}^{m}$ is matched with the firm $p_{i}$, that is $a_{j}^{m} \in \V{u}_{t}^{m}(p_{i})$, and the provided score is $y_{i,j}^{m}(t) = 1$; otherwise  $\alpha_{i,j}^{m,t+1} = \alpha_{i,j}^{m,t}$; $\beta_{i,j}^{m,t+1} =\beta_{i,j}^{m,t}+1$ if the provided score is $y_{i,j}^{m}(t) = 0$, otherwise $\beta_{i,j}^{m,t+1} = \beta_{i,j}^{m,t}$. For other unmatched pairs (firm, $\mw$), parameters are retained.

\textbf{Results.}
In Figure \ref{fig: neg regret}, we find that firms 1 and 2 achieve a total \textit{negative} sublinear regret and a total \textit{positive} sublinear regret separately (solid lines). However, we find that due to the incorrect rankings estimated for firms, firm 1 benefits from this non-optimal matching result to achieve \textit{negative} sublinear regret specifically for matching with type 1 workers often (blue dashed line). 

The occurrence of negative regret in multi-agent matching schemes presents an interesting phenomenon, contrasting the single-agent bandit problem wherein negative regret is non-existent.
In the context of the single-agent bandit problem, it is known that the best arm can be pulled, resulting in instantaneous regret that can attain zero but not take negative values. Conversely, in the multi-agent competing bandit problem, the oracle firm-optimal arm is determined by the true expected reward/utility, assuming knowledge of the true parameter $\mu^{*}$. However, due to the imprecise estimation of rankings/parameters at each time step, an exact match with the oracle policy cannot be guaranteed. This discrepancy leads to varied outcomes for firms in terms of benefits (negative instantaneous regret) or losses (positive instantaneous regret) from the matching process. Instances arise where firms may strategically submit inaccurate rankings to exploit these matches, a phenomenon termed Machiavelli/strategic behaviors. Nevertheless, over the long term, strategic actions do not yield utility gains in accordance with our policy.

\noindent
\textbf{Example 2.} We enlarge the market by expanding the DS market, particularly wanting to explore interactions between two types of workers. 
$N=2$ firms, $M=2$ types, $K_{1} = 20$ (DS) and $K_{2} = 6$  (SDE). 
The DS quota for two firms is $q_{1}^{1} = q_{2}^{1}= 1$ and the SDE quota for two firms is $q_{1}^{2} = q_{2}^{2}= 3$, and the total quota is $Q_{i} = 6$ for both firms. Preferences from firms to workers and workers to firms are still randomly generated. Therefore, the optimal matching result for each firm should consist of three workers for each type, and type II workers will be fully allocated in the first match, and the rest workers are all type II workers. All $\mmts$ initial parameters are set in the same procedure as in Example 1.

\textbf{Results.} In Figure \ref{fig: first match example 2}, we show when excessive type II workers exist, and type I workers are just right. Both firms can achieve positive sublinear regret. We find that since type II worker $K_{2} = q_{1}^{2} + q_{2}^{2} = 6$, which means in the first match stage, those type II workers are fully allocated into two firms. Thus, in the second match stage, the remaining quota would be all allocated to the type I workers for two firms. Two dotted lines represent type II regret suffered by two firms. Both firms can quickly find the type II optimal matching since finding the optimal type II match just needs the first stage of the match. However, the type I workers' matching takes a longer time to find the optimal matching (take two stages), represented by dashed lines, and both are positive sublinear regret. Therefore, these two types of matching are fully independent, which is different from Example 1.

\begin{figure}%
    \centering
    \subfigure{%
    \label{fig: neg regret}%
    \includegraphics[scale=0.22]{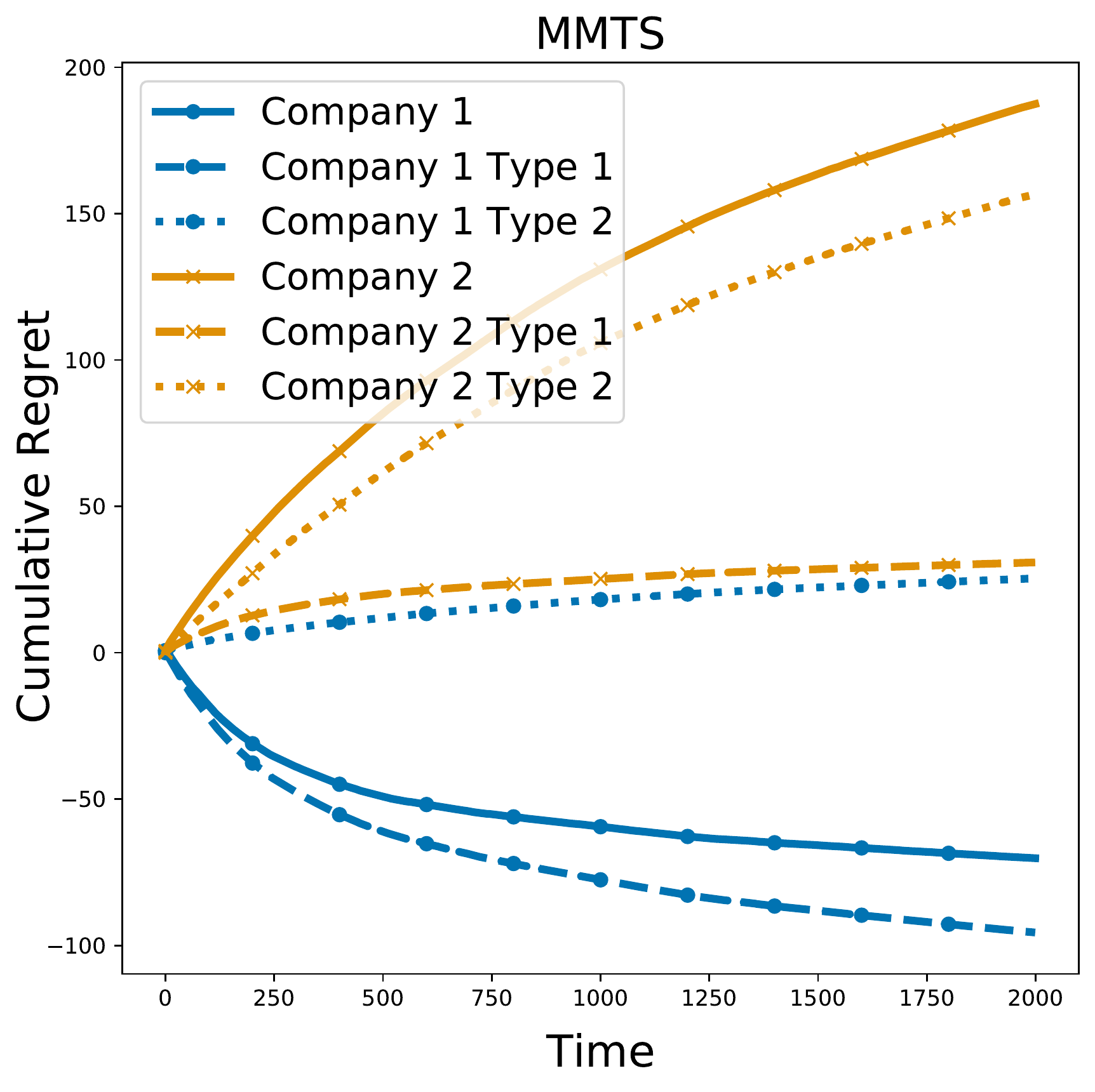}}%
    \quad
    \subfigure{%
    \label{fig: first match example 2}%
    \includegraphics[scale=0.22]{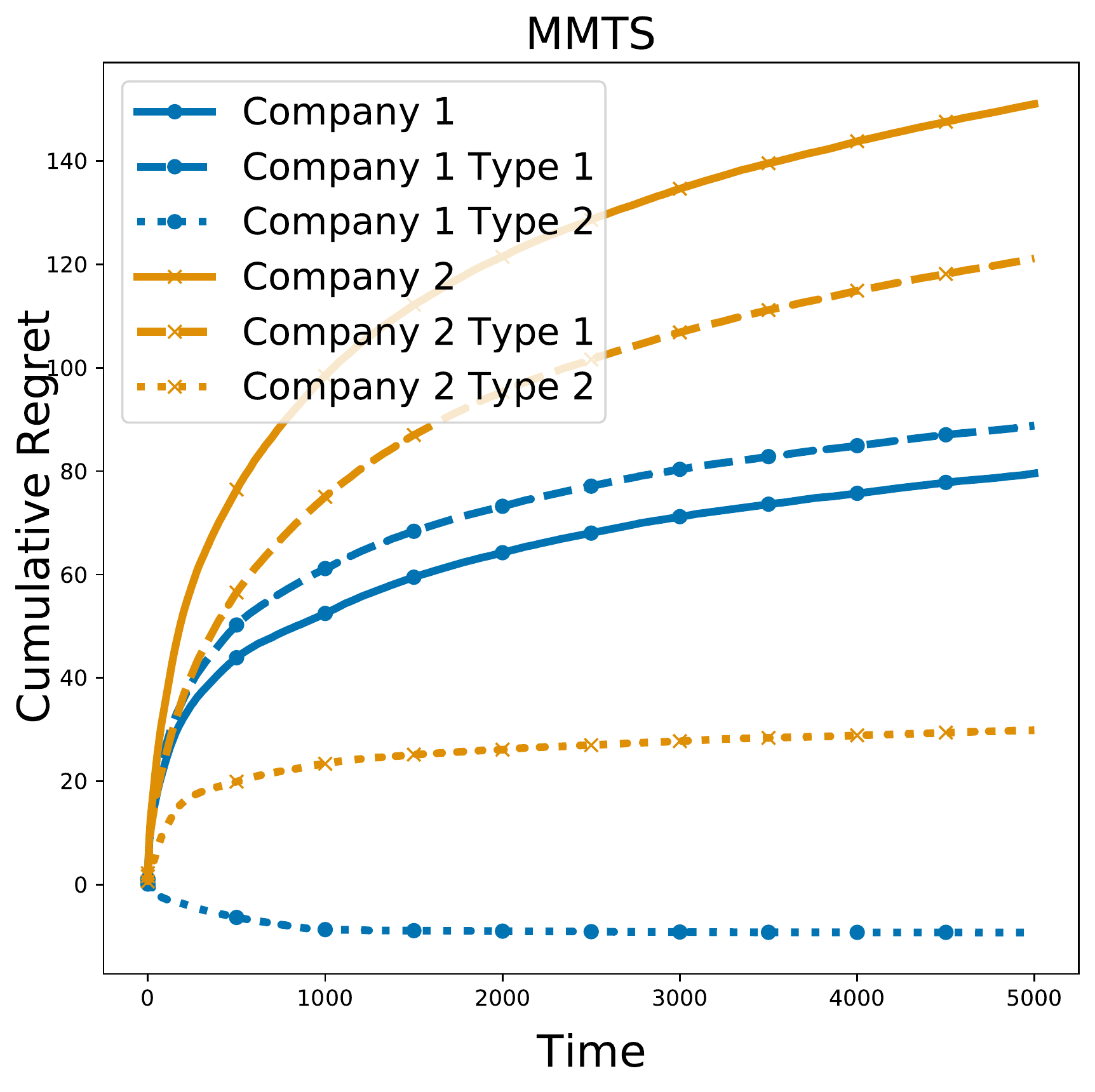}}%
    \vspace{-0.1in}
    \caption{Firms and their sub-types regret for Example 1 and, firms and their sub-types regret for Example 2.}
    \label{fig: neg and abundent}
\end{figure}
\vspace{-3mm}

\section{Related Works}
\label{sec: related work}
This section reviews two-sided matching markets with unknown preferences, multi-agent systems, assortment optimization, and matching markets.

\noindent
\textbf{Two-sided Matching Market with Unknown Preferences.} \citet{liu2020competing}  considers the multi-agent multi-armed competing problem in the centralized platform with explore-then-commit (ETC) and upper confidence bound  (UCB) style algorithms where preferences from agents to arms are unknown and need to be learned through streaming interactive data. 
\citet{jagadeesan2021learning} considers the two-sided matching problem where preferences from both sides are defined through dynamic utilities rather than fixed preferences and provide regret upper bounds over different contexts settings, and \citet{min2022learn} applies it to the Markov matching market.
\citet{cen2022regret} shows that if there is a transfer between agents, then the three desiderata (stability, low regret, and fairness) can be simultaneously achieved.
\citet{li2022rate} discusses the two-sided matching problem when the arm side has dynamic contextual information and preference is fixed from the arm side and proposes a centralized contextual ETC algorithm to obtain the near-optimal regret bound. 
Besides, there are a plethora of works discussing the two-sided matching problem in the decentralized markets \citep{liu2021bandit,basu2021beyond,sankararaman2021dominate,dai2021learningb,dai2021learninga, dai2022incentive, kong2022thompson, kong2023player, jagadeesan2022competition}. 
In particular, \citet{dai2021learninga} studies the college admission problem, provides an optimal strategy for agents, and shows its incentive-compatible property.

\textbf{Multi-Agent Systems and Game theory.}
There are some papers considering the multi-agent in sequential decision-making systems 
including the cooperative setting 
\citep{littman2001value, gonzalez2013discrete, zhang2018finite, perolat2018actor,shi2022multi}
and competing setting 
\citep{littman1994markov, auer2006logarithmic, zinkevich2007regret,wei2017online, fiez2019convergence,jin2020local}.
\citet{zhong2021can} studies the multi-player general-sum Markov games with one of the players designated as the leader and the other players regarded as followers and proposes efficient RL algorithms to achieve the Stackelberg-Nash equilibrium.

\noindent
\textbf{Assortment Optimization.}
To maximize the number of matches between the two sides (customers and suppliers), the platform must balance the inherent tension between recommending customers more potential suppliers to match with and avoiding potential collisions. \citet{ashlagi2022assortment} introduces a stylized model to study the above trade-off. 
Motivated by online labor markets \citet{aouad2022online} considers the online assortment optimization problem faced by a two-sided matching platform that hosts a set of suppliers waiting to match with a customer.
\citet{immorlica2021designing} considers a two-sided matching assortment optimization under the continuum model to achieve the optimized meeting rates and maximize the equilibrium social welfare.
\citet{rios2022improving} discusses the application of assortment optimization in dating markets. %
\citet{shi2022optimal} studies the minimum communication needed for a two-sided marketplace to reach an approximately stable outcome with the transaction price.

\noindent
\textbf{Two-sided Matching Markets with Known Preferences.}
One strand of related literature is two-sided matching, which is a stream of papers that started in \citet{gale1962college}. 
They proposed the DA algorithm with its application in the marriage problem and college admission problem.
A series of works \citep{knuth1976marriages, roth1982economics,roth1992two,roth2008deferred} discuss the theories of the DA algorithm such as stability, optimality, and incentive compatibility, and provide the practical use.
In particular, \citet{roth1985college} and \citet{sonmez1997manipulation} propose that the college admissions problem is not equivalent to the marriage problem, especially when a college can manipulate its capacity and preference.
Notably, in the hospital doctor matching example, 
since hospitals want diversity of specializations and demographic diversity, they care about the combination (group of doctors) they get.
\citet{roth1986allocation} shows that 
if all preferences are strict, and hospitals (firms) have responsive preferences, the set of doctors (workers) employed and positions filled is the same at every stable match. However, when there exists \emph{couples} in the preference list (not \emph{responsive preference} \citep{klaus2005stable}), it might make the set of stable matchings empty. Even when stable matchings exist, there need not be an optimal stable matching for either side. Later, \citet{ashlagi2011matching} revisits this couple matching problem and provides the \emph{sorted deferred acceptance algorithm} that can find a stable matching with high probability in large random markets.
\citet{biro2014hospitals} provides an integer programming model for hospital/resident problems with couples (HRC) and ties (HRCT).
\citet{manlove2017almost} releases the HRC with minimal blocking pairs and shows that if the preference list of every single resident and hospital is of length at most 2, their method can find a polynomial-time algorithm.
\citet{nguyen2018near, nguyen2022complementarities} find the stable matching in the nearby NRC problem, which is that the quota constraints are soft.
\citet{azevedo2018existence, che2019stable, greinecker2021pairwise} discuss the existence and uniqueness of stable matching with complementaries and its relationship with substitutable preferences in large economies.
Besides, there are also papers considering stability and optimality of the refugee allocation matching \citep{aziz2018stability, hadad2022improving}. 
\citet{tomoeda2018finding, boehmer2022fine} consider that firms have hard constraints both on the minimum and maximum type-specific quotas.

\section{Conclusion and Future Work}
In this paper, we proposed a new algorithm, $\mmts$ to solve the $\cmcp$. $\mmts$ builds on the strengths of TS for exploration and employs a double matching method to find a stable solution for complementary preferences and quota constraints. Through theoretical analysis, we show the effectiveness of the algorithm in achieving stability at every matching step under these constraints, achieving a $\widetilde{\mathcal{O}}(Q{\sqrt{K_{\max}T}})$-Bayesian regret over time, and exhibiting the incentive compatibility property.

There are several directions for future research.
One is to investigate more efficient exploration strategies to reduce the time required to learn the agents' unknown preferences. Another is to study scenarios where agents have indifferent preferences, and explore the optimal strategy for breaking ties. Additionally, it is of interest to incorporate real-world constraints such as budget or physical locations into the matching process, which could be studied using techniques from constrained optimization. Moreover, it is interesting to incorporate side information, such as background information of agents, into the matching process. This can be approached using techniques from recommendation systems or other machine learning algorithms that incorporate side information. Finally, it would be interesting to extend the algorithm to handle time-varying matching markets where preferences and the number of agents may change over time.

\section*{Acknowledgements}
We would like to thank the area chair and anonymous referees for their constructive suggestions
that improve the paper.
Xiaowu Dai acknowledges support of CCPR as a part of the Eunice Kennedy Shriver National Institute of Child Health and
Human Development (NICHD) population research infrastructure grant P2C-HD041022.

\section*{Impact Statement}
This paper aims to advance the two-sided matching market by addressing complementarity and unknown preferences. Our work has potential societal implications, including promoting efficient matching for couples and enhancing diversity in the matching.

\bibliography{learning}
\bibliographystyle{icml2024}

\newpage
\appendix
\onecolumn

\setcounter{equation}{0}
\renewcommand{\theequation}{C.\arabic{equation}}
\bigskip

\begin{center}
{\large\bf\MakeUppercase{SUPPLEMENT TO Two-sided Competing Matching Recommendation Markets With Quota and Complementary Preferences Constraints}}\\
\end{center}
This supplement is organized as follows. In Section \ref{supp: fea of stable matching}, we discuss the feasibility of the matching problem with complementary preference and its corresponding assumption to secure a stable matching solution. In Section \ref{sec: com complexity}, we show the computational complexity of $\mmts$.
In Section \ref{supp-sec: insufficient exploration-ucb-vs-ts}, we exhibit why the centralized UCB suffers insufficient exploration through a toy example.
In Section \ref{supp-sec:misc}, we provide the fundamental Hoeffding concentration lemma for main theorems in this paper. 
In Section \ref{supp-thm:  stability of mmts}, we provide the stability property of $\mmts$.
In Section \ref{supp-sec: regret upper bound}, we give the detailed proof of the regret upper bound of $\mmts$ and decompose its proof into three parts, regret decomposition (\ref{supp-sec: regret deco}),  bound for confidence width (\ref{lem: confidence width}), and bad events' probabilities' upper bound (\ref{supp-sec: bad event prob upper bound}). In Section \ref{supp-lem: incentive comp}, we prove $\mmts$'s strategy-proof property. Besides, as a reference, we append the DA with type and without type algorithms in Section \ref{supp-sec: algos}. Finally, in Section \ref{supp-sec: exps}, we provide details of experiments and demonstrate the robustness of $\mmts$ in large markets.

\section{Feasibility of the Stable Matching}
\label{supp: fea of stable matching}
\paragraph{Assumption for the feasibility:}
For the two-sided finite market matching problem with complementary preferences, marginal preference is a sufficient condition for the feasibility. But for the large market, it requires more assumptions such as the substitutability and indifferences, etc,.  
The key difference between the finite and infinite market matching problem \citep{azevedo2018existence, greinecker2021pairwise} lies in the agents availability.
In the infinite market, we assume that there is an uncountable number of agents on both sides of the market. This essentially means that the number of agents is so large that it can be treated as continuous, and you can't assign a specific numerical value to it. An example of an infinite market could be the matching of agents is extremely large and cannot be practically counted.
In the finite market, the number of agents on both sides is limited and countable. You can assign a specific numerical value to the number of agents. An example could be the matching of agents where there is a definite small number of agents.
In our case, in the finite market, if the complementary preference can be marginalized (or referred as the responsive preference \citep{roth1985college}, $(a_{1}, b_{1}) >(a_{1}, b_{2})$ as long as $b_{1} > b_{2}$, verse visa for $(a_{1}, b_{1}) > (a_{2}, b_{1})$ as long as $a_{1} > a_{2}$), then based on our proposed double matching algorithm and Theory 1, it exists such a stable matching solution. However, as discussed in the related works in Section \ref{sec: related work}, if there exists couples in the preference list that cannot be marginalized, which could potentially lead to an empty set of stable matchings.
\citet{che2019stable} discussed that if there exists couples in the preference list in a infinite market (large) with a continuum of workers, provided that each firm’s choice is convex and changes continuously as the set of available workers changes. They proved the \emph{existence} and structure of stable matchings under preferences exhibiting substitutability and indifferences in a large market.

\section{Complexity}
\label{sec: com complexity}
Based on \citep{gale1962college, knuth1997stable}, the stable marriage problem's DA algorithm's worst total proposal number is $N^2 - 2N + 2 = \mathcal{O}(N^2)$ when the number of participants on both sides is equal ($N=K$). The computational complexity of the college admission matching problem with quota consideration is also $\mathcal{O}(NK)$.
$\mmts$ algorithm consists of two steps of matching. The computational complexity of the first step matching is $\mathcal{O}(\sum_{m=1}^{M}NK_{m})$ if we virtually consider each type's matching process is organized in parallel. The second step's computation cost is 
also $\mathcal{O}(\sum_{m=1}^{M}NK_{m})$. That is, in the first match, if all firms are matched with their best workers, this step meets the lower bound quota constraints. Then the second match will be reduced to the standard college admission problem without type consideration and the computational complexity is $\mathcal{O}(N\sum_{m=1}^{M}K_{m})$. So the total 
computational complexity is 
still $\mathcal{O}(\sum_{m=1}^{M}NK_{m})$, which is polynomial in the number of firm ($N$) and the number of workers $\sum_{m=1}^{M}K_{m}$.

\section{Incapable Exploration}
\label{supp-sec: insufficient exploration-ucb-vs-ts}
In this section, we show why the TS strategy has an advantage over the vanilla UCB method in estimating the ranks of workers.
We even find that centralized UCB does achieve linear firm-optimal stable regret in some cases. In the following example (Example 6 from \citep{liu2020competing}), we show the firm achieves linear optimal stable regret if follow the UCB algorithm.\footnote{Here we only consider one type of worker, and the firm's quota is one.}

Let $\cN = \{p_{1}, p_{2}, p_{3}\}$, $\cK_{m} = \{a_{1}, a_{2}, a_{3}\}$, and $M=1$, with true preferences given below:
\begin{equation*}
\begin{aligned}
    &p_{1}: a_{1} \succ a_{2} \succ a_{3}   \hspace{2cm} a_{1}: p_{2} \succ p_{3} \succ p_{1}\\
    &p_{2}: a_{2} \succ a_{1} \succ a_{3}   \hspace{2cm} a_{2}: p_{1} \succ p_{2} \succ p_{3}\\
    &p_{3}: a_{3} \succ a_{1} \succ a_{2}   \hspace{2cm} a_{3}: p_{3} \succ p_{1} \succ p_{2}\\
\end{aligned}    
\end{equation*}
The firm optimal stable matching is $(p_{1}, a_{1}), (p_{2}, a_{2}), (p_{3}, a_{3})$. However, due to incorrect ranking from firm $p_{3}$, $a_{1} \succ a_{3} \succ a_{2}$, and the output stable matching is $(p_1, a_{2}), (p_{2}, a_{1}), (p_{3}, a_{3})$ based on the DA algorithm. 
In this case, $p_{3}$ will never have a chance to correct its mistake because $p_{3}$ will never be matched with $a_{1}$ again and cause the upper confidence bound for $a_{1}$ will never shrink and result in this rank $a_{1} \succ a_{3}$. Thus, it causes that $p_{1}$ and $p_{2}$ suffer linear regret.
However, the TS is capable of avoiding this situation. By the property of sampling showed in Algorithm \ref{algo:ts-sampling}, firm $p_{1}$'s initial prior over worker $a_{1}$ is a uniform random variable, and thus $r_{3}(t) > r_{1}(t)$ (if we omit $a_{2}$) with probability $\widehat{\mu}_{3} \approx \mu_{3}$, rather than \emph{zero}! This differs from the UCB style method, which cannot update $a_{1}$'s upper bound due to lacking exploration over $a_{1}$. The benefit of TS is that it can occasionally explore different ranking patterns, especially when there exists such a previous example. 
In Figure \ref{fig:ts_vs_ucb}, we show a quick comparison of centralized UCB \citep{liu2020competing} in the settings shown above and $\mmts$ when $M=1, Q=1, N=3, K=3$. The UCB method occurs a linear regret in firm 1 and firm 2 and achieves a low matching rate (0.031)\footnote{We count 1 if the matching at time $t$ is fully equal to the optimal match when two sides' preferences are known. Then we take an average over the time horizon $T$.}. However, the TS method suffers a sublinear regret in firm 1 and firm 2 and achieves a high matching rate (0.741). All results are averaged over 100 trials. See Section \ref{app: ucb vs ts} for the experimental details.

\subsection{Section \ref{sec: insufficient exploration-ucb-vs-ts} Example - Insufficient Exploration}
\label{app: ucb vs ts}

We set the true matching score for three firms to $(0.8, 0.4, 0.2)$, $(0.5, 0.7, 0.2)$, $(0.6, 0.3, 0.65)$. All preferences from companies over workers can be derived from the true matching score. As we can view, company $p_{3}$ has a similar preference over $a_{1}$ (0.6) and $a_{3}$ (0.65). Thus, the small difference can lead the incapable exploration as described in Section \ref{sec: insufficient exploration-ucb-vs-ts} by the UCB algorithm.

\section{Hoeffding Lemma}
\label{supp-sec:misc}
\begin{lem}
\label{supp: concentration}
For any $\delta > 0$, with probability $1-\delta$, the confidence width for a $\mw$ $a_{j}^{m} \in \cA_{i,t}^{m}$ at time $t$ is upper bounded by 
\begin{equation}
    w_{i, \cF_{i, t}^{m}}^{m}(a_{j}^{m}) \leq \min \bigg(2\sqrt{\frac{\log(\frac{2}{\delta})}{n_{i,j}^{m}(t)}}, 1\bigg)
\end{equation}
where $n_{i,j}^{m}(t)$ is the number of times that the pair $(p_{i}, a_{j}^{m})$ has been matched at the start of round $t$.
\end{lem}

\begin{proof}
Let $\widehat{\mu}^{m,LS}_{i,j,t} = \frac{\sum_{s=1}^{t}\V{1}(a_{j}^{m} \in \cA_{i,s}^{m})y_{i,j}^{m}(s)}{n_{i,j}^{m}(t)}$ denote the empirical mean reward from matching firm $p_{i}$ and $\mw$ $a_{j}^{m}$ up to time $t$. Define upper and lower confidence bounds as follows:
\begin{equation}
\begin{aligned}
    U_{i,t}^{m}(a_{j}^{m})
    = \min 
    \Bigg\{ 
    \widehat{\mu}^{m,LS}_{i,j,t} 
    + 
    \sqrt{
    \frac{\log(\frac{2}{\delta})}{n_{i,j}^{m}(t)}
    }, 1
    \Bigg\},
    L_{i,t}^{m}(a_{j}^{m})
    = \max 
    \Bigg\{ 
    \widehat{\mu}^{m,LS}_{i,j,t} 
    -
    \sqrt{
    \frac{\log(\frac{2}{\delta})}{n_{i,j}^{m}(t)}
    }, 0
    \Bigg\}.
\end{aligned}
\end{equation}
Then the confidence width is upper bounded by $\min \bigg(2\sqrt{\frac{\log(\frac{2}{\delta})}{n_{i,j}^{m}(t)}}, 1\bigg)$. %
\end{proof}

\section{Proof of the stability of \mmts}
\label{supp-thm:  stability of mmts}

\begin{proof}
We shall prove existence by giving an iterative procedure to find a stable matching.

\paragraph{Part I} To start, in the \emph{first match} loop, based on the double matching procedure, we can discuss $M$ types of matching in parallel. So we will only discuss the path for seeking the type-$m$ company-worker stable matching. 

Suppose firm $p_{i}$ has $q_{i}^{m}$ quota for $m$-type workers. We replace each firm $p_{i}$ by $q_{i}^{m}$ copies of $p_{i}$ denoted by $\{p_{i, 1}, p_{i, 2}, ..., p_{i, q_{i}^{m}}\}$. Each of these $p_{i, h}$ has preferences identical with those of $p_{i}$ but with a quota of 1. Further, each $m$-type worker who has $p_{i}$ on his/her preference list now replace $p_{i}$ by the set $\{p_{i, 1}, p_{i, 2}, ..., p_{i, q_{i}^{m}}\}$ in that order of preference. It is now easy to verify that the stable matchings for the firm $m$-type worker matching problem are in natural one-to-one correspondence with the stable matchings of this modified version problem. Then in the following, we only need to prove that stable matching exists in this transformed problem where each firm has quota 1, which is the standard stable marriage problem \citep{gale1962college}.
The existence of stable matching has been given in \citep{gale1962college}. Here we reiterate it to help us to find the stable matching in the \emph{second match}.

Let each firm propose to his favorite $m$-type worker. Each worker who receives more than one offer rejects all but her favorite from among those who have proposed to her. However, the worker does not fully accept the firm, but keeps the firm on a string to allow for the possibility that some better firm come along later.

Now we are in the second stage. Those firms who were rejected in the first stage propose to their second choices. Each $m$-type worker receiving 
offers chooses her favorite from the group of new firms and the firm on her string, if any. The worker rejects all the rest and again keeps the favorite in suspense. We proceed in the same manner. Those firms who are rejected at the second stage propose to their next choices, and the $m$-type workers again reject all but the best offer they have had so far.

Eventually, every $m$-type worker will have rejected a proposal, for as long as any worker has not been proposed to there will be rejections and new offers\footnote{Here we assume the number of firms is less than or equal to the number workers, and those workers unmatched finally will be matched to themselves and assume their matching object is on the firm side.}, but since no firm can propose the same $m$-type worker more than once, every worker is sure to get a proposal in due time. As soon as the last worker gets her offer, the ``recruiting" is declared over, and each $m$-type worker is now required to accept the firm on her string.

We asset that this set of matching is stable. Suppose firm $p_{i}$ and $m$-type worker $a_{j}$ are not matched to each other but firm $p_{i}$ prefers $a_{j}$ to his current matching $m$-type worker $a_{j'}$. Then $p_{i}$ must have proposed to $a_{j}$ at some stage (since the proposal is ordered by the preference list) and subsequently been rejected in favor of some firm $p_{i'}$ that $a_{j}$ liked better. It is clear that $a_{j}$ must prefer her current matching firm $p_{i'}$ and there is no instability/blocking pair.

Thus, each $m$-type firm-worker matching established on the first match is stable. Then each firm $p_{i}$'s matching object in the first match with quota $q_{i}^{m}$ can be recovered as grouping all matching objects of firm $\{p_{i, h}\}_{h=1}^{q_{i}^{m}}$.

\paragraph{Part II} To start the second match, we first check the left quota $\widetilde{Q}_{i}$ for each firm. If the left quota is zero for firm $p_{i}$, then firm $p_{i}$ and its matching workers will quit the matching market and get its stable matching object. 
Otherwise, the left firm will continue to participate in the second match. 

In the second match, preferences from firms to workers are un-categorized. Based on line 19 in Algorithm  \ref{algo:db-matching}, all types of workers will be ranked to fill the left quota. Thus, it reduces to the problem in part I, and the result matching in the second match is also stable. What is left to prove is that the overall double matching algorithm can provide stable matching. In the second match, each firm proposes to workers in his left concatenate ordered preference list, and all previous workers not in the second match preference list have already been matched or rejected. So it cannot form a blocking pair between the firm $p_{i}$ with leftover workers.
\end{proof}

\section{$\mmts$ Regret Upper Bound}
\label{supp-sec: regret upper bound}

\subsection{Regret Decomposition}
\label{supp-sec: regret deco}
In this part, we provide the roadmap of the regret decomposition and key steps to get Theorem \ref{sec: regret}.
First, we define the history for firm $p_{i}$ up to time $t$ of type $m$ as $H_{i,t}^{m}:= \{\cA_{i,1}^{m}, \V{y}_{i, \cA_{i,1}^{m}}^{m}(1), \cA_{i,2}^{m}, \V{y}_{i, \cA_{i,2}^{m}}^{m}(2),$ $..., \cA_{i,t-1}^{m}, \V{y}_{i, \cA_{i,t-1}^{m}}^{m}(t-1)\}$, composed by actions (matched workers) and rewards, where $\cA_{i,t}^{m} := \V{u}_{t}^{m}(p_{i})$ is a set of workers (based on quota requirement $q_{i}^{m}$ and $Q_{i}$) belong to $m$-type which is matched with firm $p_{i}$ at time $t$, $\V{y}_{i, \cA_{i,t-1}^{m}}^{m}(t-1)$ are realized rewards when firm $p_{i}$ matched with $\mw$s $\cA_{i,t}^{m}$. 
Define $\wtH_{i,t}:=\{H_{i,t}^{1}, H_{i,t}^{2}, ..., H_{i,t}^{M}\}$ as the aggregated interaction history between firm $p_{i}$ and all types of workers up to time $t$. 

Next, we define the \emph{good event} for firm $p_{i}$ when matching with $\mw$ at time $t$ and the true mean matching score falls in the uncertainty set as $E_{i,t}^{m} = \{\V{\mu}_{i,\cA_{i,t}^{m}}^{m} \in \cF_{i, t}^{m}\}$, where $\V{\mu}_{i,\cA_{i,t}^{m}}^{m}$ is the true mean reward vector of actually pulled arms (matched with $\mw$s) at time $t$ for firm $p_{i}$, and $\cF_{i, t}^{m}$ is the uncertainty set for $\mw$ at time $t$ for firm $p_{i}$. 
Similarly, the good event for firm $p_{i}$ when matching with all types of workers at time $t$ is $E_{i,t} = \wbcap_{m=1}^{M} E_{i,t}^{m}$, over all firms $E_{t} = \wbcap_{i=1}^{N} E_{i,t}$. And the corresponding \emph{bad event} is defined as $\overline{E}_{i,t}^{m}, \overline{E}_{i,t}, \overline{E}_{t}$ respectively. That represents the true mean vector/tensor reward of the pulled arms is not in the uncertainty set.

\begin{lem}
\label{lem: regret deco}
Fix any sequence $\{\wtcF_{i,t}: i \in [N], t \in \mathbb{N}\}$, where $\wtcF_{i,t} \subset \cF$ is measurable with respect to $\sigma(\widetilde{H}_{i,t})$. Then for any $T \in \mathbb{N}$, with probability 1,
\begin{equation}
\label{eq: regret-decom}
\begin{aligned}
    \mathfrak{R}(T) \leq 
       \bE\sum_{t=1}^{T}
       \bigg[  \sum_{i=1}^{N}\sum_{m=1}^{M}\wtW_{i,\cF_{i,t}^{m}}^{m}(\cA_{i,t}^{m})
       + 
       C\V{1}(\overline{E}_{t})
       \bigg]
\end{aligned}
\end{equation}
where $\wtW_{i,\wtcF_{i,t}^{m}}^{m}(\cdot) = \sum_{a_{j}^{m} \in \cA_{i,t}^{m}} w_{i,\wtcF_{i,t}^{m}}^{m}(a_{j}^{m})$ represents the sum of the element-wise value of uncertainty width at $\mw$ $a_{j}^{m}$. 
The uncertainty width $w_{i,\wtcF_{i,t}^{m}}^{m}(a_{j}^{m}) =\underset{\bar{\mu}^{m}_{i}, \underline{\mu}^{m}_{i} \in \wtcF_{i,t}^{m}}{\sup}(\bar{\mu}^{m}_{i}(a_{j}^{m}) - \underline{\mu}^{m}_{i}(a_{j}^{m}))$ is a worst-case measure of the uncertain about the mean reward of $\mw$ $a_{j}^{m}$. Here $C$ is a constant less than 1. 
\end{lem}

\begin{proof}
\noindent
The key step of regret decomposition is to split the instantaneous regret by firms, types, and quotas. Then we categorize regret by the happening of good events and bad events. The good events' regret is measured by the uncertainty width, and the bad events' regret is measured by the probability of happening it.

To reduce notation, define element-wise upper and lower bounds $U_{i,t}^{m}(a) = \sup\{\mu^{m}_{i}(a): \mu^{m}_{i} \in \cF_{i,t}^{m}, a \in \cK_{m} \}$ and $L_{i,t}^{m}(a) = \inf \{\mu^{m}_{i}(a): \mu^{m}_{i} \in \cF_{i,t}^{m}, a \in \cK_{m}\}$, where $\mu^{m}_{i}$ is the mean reward function $\mu^{m}_{i} \in \cF_{i,t}^{m}: \mathbb{R} \mapsto \mathbb{R}, \forall i \in [N], \forall m \in [M]$.
Whenever $\mu_{i, \wtcA_{i}^{m}}^{m} \in 
 \cF_{i,t}^{m}$, the bounds $L_{i,t}^{m}(a) \leq \mu_{i,\wtcA_{i}^{m}}^{m}(a) \leq U_{i, t}^{m}(a)$ hold for all types of workers. Here we define $\mathcal{A}_{i, t}^{m} = \V{u}^{m}_{i}(t)$ as the matched $\mw$s for firm $p_{i}$ at time $t$ and $\mathcal{A}_{i,t}^{m, *} = \overline{\V{u}}^{m}_{i}(t)$ as the firm $p_{i}$'s optimal stable matching result of $\mw$s at time $t$. Since the firm-optimal stable matching result is fixed, given both sides' preferences, we can omit time $t$ here. The firm-optimal stable matching result set is also denoted as $\mathcal{A}_{i}^{m, *} = \mathcal{A}_{i,t}^{m, *}$.

As for type-$m$ workers' matching for the firm $p_{i}$ at time $t$, the instantaneous regret with a given instance $\theta$ can be implied as follows, here for simplicity, we omit the instance conditional notation
\begin{equation}
\begin{aligned}
    \cI_{i,t}^{m} = \mu_{i }^{m}(\cA_{i}^{m,*}) - \mu_{i}^{m}(\cA_{i,t}^{m}) 
    &\leq
    \sum_{a \in \cA_{i}^{m,*}}U_{i,t}^{m}(a) - \sum_{a \in \cA_{i,t}^{m}}L_{i,t}^{m}(a) + C\V{1}(\V{\mu}_{i,\wtcA_{i}}^{m} \notin \cF_{i, t}^{m}) \\
    &= \wtU_{i, t}^{m}(\cA_{i}^{m,*}) - \wtL_{i,t}^{m}(\cA_{i,t}^{m} ) + C\V{1}(\V{\mu}_{i,\wtcA_{i}}^{m} \notin \cF_{i, t}^{m}) \\
    &=\wtW_{i,\cF_{i,t}^{m}}(\cA_{i,t}^{m}) + 
    [\wtU_{i,t}^{m}(\cA_{i}^{m,*}) - \wtU_{i,t}^{m}(\cA_{i,t}^{m})] + C\V{1}(\V{\mu}_{i,\wtcA_{i}}^{m} \notin \cF_{i, t}^{m}),
\end{aligned}
\end{equation}
where $C \leq 1$ is a constant, and we let $\wtU_{i,t}^{m}(\cdot) = \sum_{a} U_{i,t}^{m}(a)$
and $\wtW_{i,\cF_{i,t}^{m}}(\cdot) = \sum_{a} w_{i,\cF_{t}}^{m}(a)$ represent the sum of the element-wise value of $U_{i,t}^{m}(\cdot), w_{i,\cF_{i, t}}^{m}(\cdot)$, respectively.
Define the good event for firm $p_{i}$, matching with $\mw$ at time $t$ is $E_{i,t}^{m} = \{\V{\mu}_{i,\wtcA_{i}}^{m} \in \cF_{i, t}^{m}\}$,
over all types $E_{i,t} = \wbcap_{m=1}^{M} E_{i,t}^{m}$, over all firms $E_{t} = \wbcap_{i=1}^{N} E_{i,t}$. And the corresponding bad event is defined as $\overline{E}_{i,t}^{m}, \overline{E}_{i,t}, \overline{E}_{t}$ respectively.

Now consider Eq.~\eqref{eq: regret-decom}, summing over the previous equation over time $t$, firms $p_{i}$, and workers' type $m$, we get
\begin{equation} 
    \begin{aligned}
        \mathfrak{R}(T) 
        &\leq 
        \bE \sum_{i=1}^{N}\sum_{t=1}^{T}
        \sum_{m=1}^{M}[\wtW_{i,\cF_{i,t}^{m}}(\cA_{i,t}^{m}) + C\V{1}(\overline{E}_{t})] + \sum_{i=1}^{N}\mathbb{E}M_{i,T}\\
       &= \bE\sum_{t=1}^{T}
       [ C\V{1}(\overline{E}_{t}) 
       +  \sum_{i=1}^{N}\sum_{m=1}^{M}\wtW_{i,\cF_{i,t}^{m}}(\cA_{i,t}^{m})] + \sum_{i=1}^{N}\mathbb{E}M_{i,T}
    \end{aligned}
\end{equation}
where $M_{i,T} = \sum_{t=1}^{T}\sum_{m=1}^{M} [\wtU_{i,t}^{m}(\cA_{i}^{m,*}) - \wtU_{i,t}^{m}(\cA_{i,t}^{m})]$. Now by the definition of TS, $\mathbb{P}_{m}(\cA_{i,t}^{m} \in \cdot|H_{i,t}^{m}) = \mathbb{P}_{m}(\cA_{i}^{m,*} \in \cdot|H_{i,t}^{m})$ for all types, where $\mathbb{P}_{m}( \cdot |H_{i,t}^{m})$ represents this probability is conditional on history $H_{i,t}^{m}$ and the selected action (worker) belongs in $m$-type workers for firm $p_{i}$. That is $\cA_{i,t}^{m}$ and $\cA_{i}^{m,*}$ within type-$m$ is identically distributed under the posterior. Besides, since the confidence set $\mathcal{F}_{i,t}^{m}$ is $\sigma(H_{i,t}^{m})$-measurable, so is the induced upper confidence bound $U_{i,t}^{m}(\cdot)$. This implies $\mathbb{E}_{m}[U_{i,t}^{m}(\cA_{i,t}^{m})| H_{t}^{m}] = \mathbb{E}_{m}[U_{i,t}^{m}(\cA_{i}^{m,*})| H_{t}^{m}]$, and there for $\mathbb{E}[M_{i,T}] = 0$ and $\sum_{i=1}^{N}\mathbb{E}M_{i,T} = 0$. Then we can obtain the desired result.
\end{proof}

\subsection{Uncertainty Widths}
In this part, we provide the upper bound of the accumulated uncertainty widths over all types of workers and  all firms, which is the first part in Eq. \eqref{eq: regret-decom}.
\begin{lem}
\label{lem: confidence width}

If $(\beta_{i,j,t}^{m} \geq 0 | t \in \bN)$ is a non-decreasing sequence and $\mathcal{F}_{i,j,t}^{m}:= \{\mu_{i, j}^{m} \in \cF_{i,j}^{m}: \norm{\mu_{i,j}^{m} - \widehat{\mu}^{m,LS}_{i,j,t}}_{1} \leq \sqrt{\beta_{i,j,t}^{m}}\}$, then with probability 1, 
\begin{equation*}
\sum_{t=1}^{T}\sum_{i=1}^{N}\sum_{m=1}^{M}\wtW_{i,\cF_{i,t}^{m}}^{m}(\cA_{i,t}^{m})
        \leq 
         8Q\log(QT) \sqrt{K_{\max}T}.
\end{equation*}
\end{lem}

\noindent
The proof of this lemma builds upon Lemma \ref{lem: eluder dim}, which establishes the number of instances where the widths of uncertainty sets for a chosen set of $\mw$s $\cA_{i,t}^{m}$ greater than $\epsilon$. We show that this number is determined by the \emph{Eluder dimension} \citep{russo2014learning}. 

\begin{proof}
By Lemma \ref{lem: regret deco}, the instantaneous regret $\mathcal{I}_{t}$ over all firms and all types, can be decomposed by types and by firms and shown as  
\begin{equation}
\label{eq: regret to concen}
\begin{aligned}
    \mathcal{I}_{t} 
    &= \sum_{m=1}^{M}\mathcal{I}_{t}^{m} = \sum_{i=1}^{N}\sum_{m=1}^{M} \mathcal{I}_{i, t}^{m}\\
    &\leq 
    \sum_{i=1}^{N}\sum_{m=1}^{M} \wtW_{i,\cF_{i,t}^{m}}(\cA_{i,t}^{m}), \quad \text{if } E_{t} \text{ holds.}  \\
    &\leq 2\sum_{i\in [N], m \in [M], a_{j}^{m} \in \cK_{m}} \sqrt{\frac{\log(\sum_{i=1}^{N}Q_{i}T)}{n_{i,j}^{m}(t)}}, \quad \text{with prob } 1-\delta
\end{aligned}    
\end{equation}
where the first inequality is based on Lemma \ref{lem: regret deco} and if $E_{t}$ holds for $t \in \bN, m \in M, i \in [N]$, $n_{i,j}^{m}(t)$ is the number of times that the pair $(p_{i}, a_{j}^{m})$ has been matched at the start of round $t$. The second inequality is constructed from a union concentration inequality based on Lemma \ref{supp: concentration}, and we set $\delta = 2/\sum_{i=1}Q_{i}T$.
We denote $z_{i,j}^{m}(t) = \frac{1}{\sqrt{n_{i,j}^{m}(t)}}$ as the size of the scaled confidence set (without the log factor) for the pair $(p_{i}, a_{j}^{m})$ at the time $t$.

At each time step $t$, let's consider the list consisting of $z_{i,j}^{m}(t)$ and reorder the overall list consisting of concatenating all those scaled confidence sets over all rounds and all types  in decreasing order. Then we obtain a list $\tilde{z}_{1} \geq \tilde{z}_{2} \geq ..., \geq \tilde{z}_{L}$, where $L = \sum_{t=1}^{T}\sum_{i=1}^{N}Q_{i} = T\sum_{i=1}^{N}Q_{i}$. We reorganize the Eq. \eqref{eq: regret to concen} to get
\begin{equation}
\label{eq: all regret}
    \begin{aligned}
        \sum_{t=1}^{T}\cI_{t} \leq \sum_{t=1}^{T} \sum_{m=1}^{M}\sum_{i=1}^{N} 
        \wtW_{i,\cF_{i,t}^{m}}(\cA_{i,t}^{m}) \leq 2\log(\sum_{i=1}^{N}Q_{i}T)\sum_{l=1}^{L}\tilde{z}_{l}.      
    \end{aligned}
\end{equation}
By Lemma \ref{lem: eluder dim}, the number of rounds that a pair of a firm and any $\mw$ can have it confidence set have size at least $\tilde{z}_{l}$ is upper bounded by $(1 + \frac{4}{\tilde{z}_{l}^{2}})K_{m}$ when we set $\epsilon = \tilde{z}_{l}$ and know $\beta_{i,j,t}^{m} \leq 1$. Thus, the total number of times that any confidence set can have size at least $\tilde{z}_{l}$ is upper bounded by $\big(1+\frac{4}{\tilde{z}_{l}^{2}} \big)\sum_{i=1}^{N}\sum_{m=1}^{M}|\cA_{i,t}^{m}|K_{m}$. To determine the minimum condition for $\tilde{z}_{l}$, which is equivalent to determine the maximum of $l$, we have $l \leq \big(1+\frac{4}{\tilde{z}_{l}^{2}} \big)\sum_{i=1}^{N}\sum_{m=1}^{M}|\cA_{i,t}^{m}|K_{m}$.
So we claim that
\begin{equation}
    \begin{aligned}
        \tilde{z}_{l} \leq \min \Bigg(1, \frac{2}{\sqrt{\frac{l}{\sum_{i=1}^{N}\sum_{m=1}^{M}|\cA_{i,t}^{m}|K_{m}}-1}}\Bigg)
        \leq 
        \min \Bigg(1, \frac{2}{\sqrt{\frac{
        l}{\sum_{i=1}^{N}Q_{i}K_{\max}}-1}}\Bigg),
    \end{aligned}    
\end{equation}
where the second inequality above is by $\sum_{i=1}^{N}\sum_{m=1}^{M}|\cA_{i,t}^{m}|K_{m} \leq K_{\max} \sum_{i=1}^{N}\sum_{m=1}^{M}|\cA_{i,t}^{m}| \leq K_{\max}\sum_{i=1}^{N}Q_{i} = QK_{\max}$ and $K_{\max} = \max\{K_{1},..., K_{M}\}, Q = \sum_{i=1}^{N}Q_{i}$. Putting all these together, we have
\begin{equation}
    \begin{aligned}
        2\log(\sum_{i=1}^{N}Q_{i}T)
        \sum_{l=1}^{L}\tilde{z}_{l}
        &\leq
        2\log(QT) \sum_{l=1}^{L} \min (1, \frac{2}{\sqrt{\frac{l}{QK_{\max}}-1}})\\
        &=4\log(QT) \sum_{l=1}^{QT}
           \frac{1}{\sqrt{\frac{l}{QK_{\max}}-1}} \\
        &\leq 
        8\log(QT) \sqrt{QK_{\max}}\sqrt{QT}
    \end{aligned}
\end{equation}
where the last inequality is by intergral inequality $$\sum_{l=1}^{QT}\frac{1}{\sqrt{\frac{l}{QK_{\max}}-1}} \leq \sqrt{QK_{\max}}\sum_{l=1}^{QT}\frac{1}{\sqrt{l}}\leq \sqrt{QK_{\max}} \int_{x=0}^{QT}\frac{1}{\sqrt{x}}dx = 2\sqrt{QK_{\max}}\sqrt{QT}.$$
Based on Eq.~\eqref{eq: all regret} and the above result, we can get the regret
\begin{equation}
    \begin{aligned}
        \sum_{t=1}^{T}\cI_{t} \leq  8Q\log(QT) \sqrt{K_{\max}T}, 
    \end{aligned}
\end{equation}
if $E_{t}$ holds.
\end{proof}

\begin{lem}
\label{lem: eluder dim}
If $(\beta_{i,j,t}^{m}\geq 0 | t \in \mathbb{N})$ is a nondecreasing sequence for $i\in[N], a_{j}^{m} \in \cK_{m}, m \in [M]$ and $\mathcal{F}_{i,j,t}^{m}:= \{\mu_{i, j}^{m} \in \cF_{i,j}^{m}: \norm{\mu_{i,j}^{m} - \widehat{\mu}^{m,LS}_{i,j,t}}_{1} \leq \sqrt{\beta_{i,j,t}^{m}}\}$, 
for all $T \in \bN$ and $\epsilon > 0$,
then 
\begin{equation*}
\sum_{t=1}^{T} \sum_{m=1}^{M}
    \sum_{a_{j}^{m}\in \cA_{i,t}^{m}}
    \V{1}\big(w_{i, \cF_{i, t}^{m}}^{m}
    (a_{j}^{m}) > \epsilon\big) 
    \leq 
    \big(\frac{4 \wtbeta_{i,T}}{\epsilon^{2}} +1\big)
     \sum_{m=1}^{M}
    |\cA_{i,t}^{m}|
    K_{m}.
\end{equation*}
Here $\widehat{\mu}^{m, LS}_{i,j,t} = \frac{\sum_{s=1}^{t}\V{1}(a_{j}^{m} \in \cA_{i,s}^{m})y_{i,j}^{m}(s)}{n_{i,j}^{m}(t)}$ is the estimated average reward for $\mw$ $a_{j}^{m}$ from the view point of firm $p_{i}$ at time $t$, and $n_{i,j}^{m}(t)$ is the number of matched times up to time $t$ of firm $p_{i}$ with $\mw$ $a_{j}^{m}$. Besides, we define $\wtbeta_{i,T} = \underset{a_{j}^{m} \in \cK_{m}, m\in [M]}{\max}\beta_{i,j,T}^{m}$ as the maximum uncertainty bound over all types of workers at time $T$ for firm $p_{i}$.
\end{lem}

\noindent
The proof of this result is based on techniques from \citep{russo2013eluder, russo2014learning}. This result demonstrates that the upper bound of the number of times the widths of uncertainty sets exceeds $\epsilon$ is dependent on the error $\cO{(\epsilon^{-2})}$ and linearly proportional to the product of the number of $\mw$ and the type quota size $q_{i}^{m}$.

\begin{proof}
Based on the Proposition 3 from \citep{russo2013eluder}, we can use the \emph{eluder dimension} $dim_{E}(\cF_{i}^{m}, \epsilon)$ to bound the number of times the widths of confidence intervals for a selection of set of $\mw$s $\cA_{i,t}^{m}$ greater than $\epsilon$.
\begin{equation}
    \begin{aligned}
    \sum_{t=1}^{T} 
    \sum_{m=1}^{M}
    \sum_{a_{j}^{m}\in \cA_{i,t}^{m}}
    \V{1}\bigg(w_{i,  \cF_{i, t}^{m}}^{m}
    (a_{j}^{m}) > \epsilon\bigg) 
    &\leq 
    \sum_{m=1}^{M}
    \sum_{a_{j}^{m}\in \cA_{i,t}^{m}}
    \bigg(\frac{4\beta_{i,j,T}^{m}}{\epsilon^{2}} +1\bigg)\text{dim}_{E}(\cF_{i}^{m}, \epsilon)\\
    &\leq
    \Bigg(\frac{4 \underset{a_{j}^{m} \in \cK_{m}, m\in [M]}{\max}\beta_{i,j,T}^{m}}{\epsilon^{2}} +1\Bigg)\sum_{m=1}^{M}
    |\cA_{i,t}^{m}|
    \text{dim}_{E}(\cF_{i}^{m}, \epsilon),
    \end{aligned}
\end{equation}
where the eluder dimension of a multi-arm bandit problem is the number of arms, we get
\begin{equation}
    \begin{aligned}
    \sum_{t=1}^{T} 
    \sum_{m=1}^{M}
    \sum_{a_{j}^{m}\in \cA_{i,t}^{m}}
    \V{1}\bigg(w_{i, \cF_{t}}^{m}
    (a_{j}^{m}) > \epsilon\bigg) 
    &\leq 
     \Bigg(\frac{4 \wtbeta_{i,T}}{\epsilon^{2}} +1\Bigg)
     \sum_{m=1}^{M}
    |\cA_{i,t}^{m}|
    K_{m}
    \leq \Bigg(\frac{4 \wtbeta_{i,T}}{\epsilon^{2}} +1\Bigg) Q_{i}K_{\max}
    \end{aligned}
\end{equation}
where $\wtbeta_{i,T} = \underset{a_{j}^{m} \in \cK_{m}, m\in [M]}{\max}\beta_{i,j,T}^{m}$. Besides, we know that $Q_{i} = \sum_{m=1}^{M}|\cA_{i,t}^{m}|$ and define $K_{\max} = \underset{m \in[M]}{\max} K_{m}$, so we can get the second inequality.
\end{proof}

\subsection{Bad Event Upper Bound}
\label{supp-sec: bad event prob upper bound}
In this part, we provide an upper bound of the second part of Eq. \eqref{eq: regret-decom}.
The regret caused by the happening of the bad event at each time step is quantified by the following lemma.
\begin{lem}
\label{lem: bad event prob}
If $\mathcal{F}_{i,j,t}^{m}:= \{\mu_{i, j}^{m} \in \cF_{i,j}^{m}: \norm{\mu_{i,j}^{m} - \widehat{\mu}^{m,LS}_{i,j,t}}_{1} \leq \sqrt{\beta_{i,j,t}^{m}}\}$ holds with probability $1-\delta$, then the bad event $\overline{E}_{t}$ happening's probability is upper bounded by $\bE \V{1}(\overline{E}_{t}) \leq NK\delta$.
In particular, if $\delta = 1/QT$, the accumulated bad events' probability is upper bounded by $\sum_{t=1}^{T}\bE \V{1}(\overline{E}_{t}) \leq NK/Q$.
\end{lem}

\noindent
To bound the probability of bad events, we use a union bound to obtain the desired result. Specifically, if $Q_{i} = 1$, which means each firm has a total quota of 1 and only considers one type of worker, then $\sum_{t=1}^{T}\bE \V{1}(\overline{E}_{t}) \leq NK/(N\times1) = K$. This shows that each firm needs to explore a single type of worker, and the worst total regret is less than $K$. If $Q_{i} = 1, M=1$, which means all firms have the same recruiting requirements, the result reduces to the general competitive matching scenario, and the worst regret is the number of workers of type $K_{M}$ in the market.

\begin{proof} 
If $E_{t}$ does not hold, the probability of the true matching score is not in the confidence interval we constructed is upper bounded by 
\begin{equation}
    \begin{aligned}
        \bE \V{1}(\overline{E}_{t})
        &=
        \bP(\overline{E}_{t})
        = 
        \bP\Bigg(\big(\wbcap_{i \in [N]}\wbcap_{m \in  [M]} \wbcap_{a_{j}^{m}\in \cK_{m}}
        \{\mu_{i, j}^{m} \in \cF_{i,j,t}^{m}\}\big)^{c}\Bigg)\\
        &=
        \bP\bigg(
        \wbcup_{i\in [N]}\wbcup_{a_{j}^{m} \in \cK_{m}} \wbcup_{m \in [M]}
        \{\mu_{i, j}^{m} \notin \cF_{i,j,t}^{m}\}
        \bigg)\\
        &=
        \bP\bigg(
        \wbcup_{i\in [N]}\wbcup_{a_{j}^{m} \in \cK_{m}} \wbcup_{m \in [M]}
        \bigg\{\norm{\mu_{i,j}^{m} - \widehat{\mu}^{m,LS}_{i,j,t}}_{2, E_{t}} \geq \sqrt{\beta_{i,j,t}^{m}}\bigg\}
        \bigg)\\
        &=
        \bP\bigg(
        \wbcup_{i\in [N]}\wbcup_{a_{j}^{m} \in \cK_{m}} \wbcup_{m \in [M]}
        \bigg\{\norm{\mu_{i,j}^{m} - \widehat{\mu}^{m,LS}_{i,j,t}}_{1} \geq \sqrt{\frac{\log(\frac{2}{\delta})}{n_{i,j}^{m}(t)}}\bigg\}
        \bigg)\\
        &\leq 
        \sum_{i\in [N]}\sum_{a_{j}^{m} \in \cK_{m}} \sum_{m \in [M]}\bP\bigg(
        \norm{\mu_{i,j}^{m} - \widehat{\mu}^{m,LS}_{i,j,t}}_{1} \geq \sqrt{\frac{\log(\frac{2}{\delta})}{n_{i,j}^{m}(t)}}
        \bigg)\\
    \end{aligned}
\end{equation}
where the third equality is by De-Morgan's Law of sets. 
In the last inequality, we use the union bound to control the probability. Since each $\widehat{\mu}_{i,j}^{m, LS} - \mu_{i,j}^{m}$ is a mean zero and $\frac{1}{2n_{i,j}^{m}}$-sub-Gaussian random variable, based on Lemma \ref{supp: concentration}, have $\bP(\norm{\mu_{i,j}^{m} - \widehat{\mu}^{m,LS}_{i,j,t}}_{1} \geq \sqrt{\frac{\log(\frac{2}{\delta})}{n_{i,j}^{m}(t)}})\leq \delta$. The overall bad event's probability's upper bound is 
\begin{equation}
    \begin{aligned}
        \bP(\overline{E}_{t}) 
        \leq
        N K \delta
    \end{aligned}
\end{equation}
Based on our confidence width is less than 1, so $C=1, \forall i \in [N]$. The expected regret from this bad event is not in the confidence interval at most
\begin{equation}
    \begin{aligned}
            N K \delta \cdot C T 
            & \leq  N K \frac{1}{\sum_{i=1}^{N}Q_{i}T}  T = \frac{NK}{Q}
    \end{aligned}
\end{equation}
This part's regret is negligible compared with the regret from Lemma \ref{lem: confidence width}. In particular, if there is only one type and each firm has only one position to be filled. Thus, $Q = N$, the bad event's upper bounded probability will shrink to $K$, the number of workers to be explored.
\end{proof}

In this part, we provide the proof of $\mmts$'s Bayesian regret upper bound.
\subsection{Proof of Theorem \ref{thm: regret upper bound}}
\begin{thm}
\label{supp-thm: regret upper bound}
When all firms follow the $\mmts$ algorithm, the platform will incur the Bayesian total expected regret
\begin{equation}
    \mathfrak{R}(T) \leq 8\log(QT) \sqrt{QK_{\max}}\sqrt{QT} + NK/Q
\end{equation}
where $K_{\max} = \max\{K_{1},..., K_{M}\}, K = \sum_{m=1}^{M}K_{m}$ .
\end{thm}
\begin{proof}
We decompose the Bayesian Social Welfare Gap for all firms by 
\begin{equation}
    \begin{aligned}
        \mathfrak{R}(T) &= \mathbb{E}_{\theta \in \Theta}\bigg[\sum_{i=1}^{N}R_{i}(T,\theta)\bigg] 
        = 
        \mathbb{E}_{\theta \in \Theta}\bigg[ \sum_{i=1}^{N}\sum_{m=1}^{M}\sum_{t=1}^{T}\mu_{i,\overline{u}_{i}^{m}(t)}(t) - \sum_{i=1}^{N}\sum_{m=1}^{M}\sum_{t=1}^{T} \mu_{i,u_{i}^{m}}(t)|\theta\bigg]\\
        &= \sum_{i=1}^{N}\sum_{t=1}^{T}\mathbb{E}_{\theta \in \Theta}\bigg[ \sum_{m=1}^{M}
        (\mu_{i,\overline{u}_{i}^{m}(t)}(t) - \mu_{i,u_{i}^{m}}(t))
        |\theta\bigg]\\
        &=
        \mathbb{E}_{\theta \in \Theta}\bigg[\sum_{t=1}^{T}
        \sum_{i=1}^{N}\sum_{m=1}^{M}
        \cI_{i,t}^{m}
        |\theta\bigg]\\
        &=\mathbb{E}_{\theta \in \Theta}
        \bigg[\sum_{t=1}^{T}
        \cI_{t}
        |\theta\bigg]
    \end{aligned}
\end{equation}
where we define $\cI_{i,t}^{m} = \mu_{i,\V{\theta}}^{m}(\cA_{i}^{m,*}) - \mu_{i, \V{\theta}}^{m}(\cA_{i,t}^{m})$ and $\cI_{t} = \sum_{i=1}^{N}\sum_{m=1}^{M}\cI_{i,t}^{m}$. Here $\cA_{i}^{m,*}$ is the optimal matched workers for firm $p_{i}$ of type $m$ and $\cA_{i,t}^{m}$ is the actual matched workers for firm $p_{i}$ of type $m$ at time $t$ under the instance $\theta$.
 
Based Lemma \ref{lem: regret deco}, $\mathfrak{R}(T)$ is upper bounded by $\bE\sum_{t=1}^{T}\big[ C\V{1}(\overline{E}_{t}) +  \sum_{i=1}^{N}\sum_{m=1}^{M}\wtW_{i,\cF_{i,t}^{m}}(\cA_{i,t}^{m})\big]$. The first term, the sum of the bad event probability $\bE\sum_{t=1}^{T}C\V{1}(\overline{E}_{t}) = C\sum_{t=1}^{T}\bP(\overline{E}_{t})$, which is upper bounded by $NK/Q$ based on Lemma \ref{lem: bad event prob} and $C\leq 1$. The second term, the sum of confidence widths is upper bounded by $8Q\log(QT) \sqrt{TK_{\max}}$ based on Lemma \ref{lem: confidence width}. Thus the Bayesian regret is upper bounded by $8Q\log(QT) \sqrt{TK_{\max}} + NK/Q$.
\end{proof}

\section{Incentive-Compatibility}
\label{sec: app-strategy}
In this section, we discuss the incentive-compatibility property of $\mmts$. That is, if one firm does not follow the $\mmts$ when all other firms submit their $\mmts$ preferences, that firm cannot benefit (matched with a better worker than his optimal stable matching worker) over a sublinear order. 
As we know, 
\citet{dubins1981machiavelli}
discussed the \emph{Machiavelli} firm could not benefit from incorrectly stating their true preference when there exists a unique stable matching. However, when one side's preferences are unknown and need to be learned through data, this result no longer holds. Thus, the maximum benefits that can be gained by the Machiavelli firm are under-explored in the setting of learning in matching. \citet{liu2020competing} discussed the benefits that can be obtained by Machiavelli firm when other firms follow the centralized-UCB algorithm with the problem setting of one type of worker and quota equal one in the market.

We now show in $\cmcp$,  when all firms except one $p_{i}$ submit their $\mmts$-based preferences to the 
matching platform, the firm $p_{i}$ has an incentive also to submit preferences based on their sampling rankings in a \emph{long horizon}, so long as the matching result do not have multiple stable solutions. Now we establish the following lemma, which is an upper bound of the expected number of pulls that a firm $p_{i}$ can match with a $m$-type worker that is better than their optimal $m$-type workers, regardless of what preferences they submit to the platform.

Let's use $\cH_{i,l}^{m}$ to define the achievable \emph{sub-matching} set of  $\V{u}^{m}$ when all firms follow the $\mmts$, which represents firm $p_{i}$ and $\mw$ $a_{l}^{m}$ is matched such that $a_{l}^{m} \in \V{u}_{i}^{m}$. Let \textUpsilon$_{\V{u}^{m}}(T)$ be the number of times sub-matching $\V{u}^{m}$ is played by time $t$. We also provide the blocking triplet in a matching definition as follows.

\begin{defn}[Blocking triplet]
A blocking triplet $(p_{i}, a_{k}, a_{k'})$ for a matching $u$ is that there must exist a firm $p_{i}$ and worker $a_{j}$ that they both prefer to match with each other than their current match. That is, if $a_{k'} \in \V{u}_{i}$, $\mu_{i, k'} < \mu_{i, k}$ and worker $a_{k}$ is either unmatched or $\pi_{k,i} < \pi_{k, \V{u}^{-1}(k)}$.  
\end{defn}

The following lemma presents the upper bound of the number of matching times of $p_{i}$ and $a_{l}^{m}$ by time $T$, where $a_{l}^{m}$ is a \emph{super optimal} $\mw$ (preferred than all stable optimal $\mw$s under true preferences), when all firms follow the $\mmts$.

\begin{lem}
\label{lem: app-IC lemma}
Let \textUpsilon$_{i,l}^{m}(T)$ be the number of times a firm $p_{i}$ matched with a $m$-type worker such that the mean reward of $a_{l}^{m}$ for firm $p_{i}$ is greater than $p_{i}$'s optimal match $\overline{\V{u}}_{i}^{m}$, which is $\mu_{i,a_{l}^{m}}^{m} > \underset{a_{j}^{m} \in \overline{\V{u}}_{i}^{m}}{\max}\mu_{i, j}^{m}$. Then the expected number of matches between $p_{i}$ and $a_{l}^{m}$ is upper bounded by
\begin{equation*}
\bE[\text{\textUpsilon}_{i,l}^{m}(T)] \leq 
        \underset{S^{m} \in \cC(\cH_{i,l}^{m})}{\min}
        \sum_{(p_{j}, a_{k}^{m}, a_{k'}^{m})\in S^{m}}
        \big(
            C^{m}_{i,j,k'}(T) +
            \frac{\log(T)}{d(\mu_{j,\overline{\V{u}}_{i,\min}^{m}}, \mu_{j, k'})}
        \big),
\end{equation*}
where $\overline{\V{u}}_{i,\min}^{m} = \underset{a_{k}^{m} \in \overline{\V{u}}_{j}^{m}}{\argminE}\hspace{1mm} \mu_{i,k}^{m}$, and  $C^{m}_{i,j,k'} = \cO((\log(T))^{-1/3})$.
\end{lem}

Then we provide the benefit (lower bound of the regret) of Machiavelli firm $p_{i}$ can gain by not following the $\mmts$ from matching with $m$-type workers. Let's define the \emph{super worker reward gap} as $\overline{\Delta}_{i,l}^{m} = \underset{a_{j}^{m} \in \overline{\V{u}}_{i}^{m}}{\max}\mu_{i, j}^{m} - \mu_{i, l}^{m}$, where $a_{l}^{m} \notin \overline{\V{u}}_{i}^{m}$.
\begin{thm}
\label{thm: app-IC}
Suppose all firms other than firm $p_{i}$ submit preferences according to the $\mmts$ to the centralized platform. Then the following upper bound on firm $p_{i}$'s optimal regret for $m$-type workers holds:
\begin{equation}
    \begin{aligned}
        R_{i}^{m}(T,\theta) \geq \sum_{l: \overline{\Delta}_{i,l}^{m} < 0} \overline{\Delta}_{i,l}^{m}
        \Bigg[
        \underset{S^{m} \in \cC(\cH_{i,l}^{m})}{\min}
        \sum_{(p_{j}, a_{k}^{m}, a_{k'}^{m})\in S^{m}}
        \bigg(
            C^{m}_{i,j,k'} +
            \frac{\log(T)}{d(\mu_{j, \overline{\V{u}}_{i,\min}^{m}}, \mu_{j, k'})}
        \bigg)
        \Bigg]
    \end{aligned}
\end{equation}
where $\overline{\V{u}}_{i,\min}^{m} = \underset{a_{k}^{m} \in \overline{\V{u}}_{j}^{m}}{\argminE}\hspace{1mm} \mu_{i,k}^{m}$, and  
$C^{m}_{i,j,k'} = \cO((\log(T))^{-1/3})$.
\end{thm}

This result can be directly derived from Lemma \ref{lem: IC lemma}.
Theorem \ref{thm: IC} demonstrates that there is no sequence of preferences that a firm can submit to the centralized platform that would result in negative optimal regret greater than $\cO(\log T)$ in magnitude within type $m$. When considering multiple types together for firm $p_{i}$, this magnitude remains $\cO(\log T)$ in total. Theorem \ref{thm: IC} confirms that, when there is a unique stable matching in type $m$, firms cannot gain significant advantage in terms of firm-optimal stable regret by submitting preferences other than those generated by the $\mmts$ algorithm. An example is provided in Section \ref{exp: negative regret} to illustrate this incentive compatibility property.
Figure \ref{fig: neg regret} illustrates the total regret, with solid lines representing the aggregate regret over all types for each firm, and dashed lines representing the regret for each type. It is observed that the type 1 regret of firm 1 is negative, owing to the inaccuracies in the rankings submitted by both firm 1 and firm 2. A detailed analysis of this negative regret pattern is given in Section \ref{sec: learning parameters}.

\subsection{Proof of Incentive Compatibility}
\label{supp-lem: incentive comp}

\begin{lem}
\label{supp-lem: IC lemma}
Let \textUpsilon$_{i,l}^{m}(T)$ be the number of times a firm $p_{i}$ matched with a $m$-type worker such that the mean reward of $a_{l}^{m}$ for firm $p_{i}$ is greater than $p_{i}$'s optimal match $\overline{u}_{i}^{m}$, which is $\mu_{i,a_{l}^{m}}^{m} > \underset{a_{j}^{m} \in \overline{u}_{i}^{m}}{\max}\mu_{i, j}^{m}$. Then
\begin{equation}
    \begin{aligned}
        \bE[\text{\textUpsilon}_{i,l}^{m}(T)] \leq 
        \underset{S^{m} \in \cC(\cH_{i,l}^{m})}{\min}
        \sum_{(p_{j}, a_{k}^{m}, a_{k'}^{m})\in S^{m}}
        \bigg(
            C^{m}_{i,j,k'}(T) +
            \frac{\log(T)}{d(\mu_{j,\overline{u}_{i,\min}^{m}}, \mu_{j, k'})}
        \bigg)
    \end{aligned}
\end{equation}
where $\overline{u}_{i,\min}^{m} = \underset{a_{k}^{m} \in \overline{u}_{j}^{m}}{\argminE}\hspace{1mm} \mu_{i,k}^{m}$,  $C^{m}_{i,j,k'} = \cO((\log(T))^{-1/3})$.
\end{lem}

\begin{proof}

We claim that if firm $p_{i}$ is matched with a \emph{super optimal} $\mw$ $a_{l}^{m}$ in any round, the matching $u^{m}$ must be unstable according to true preferences from both sides. We then state that there must exist a $m$-type blocking triplet $(p_{j}, a_{k}^{m}, a_{k'}^{m})$ where $p_{j} \neq p_{i}$.

We prove it by contradiction. Suppose all blocking triplets in matching $u$ \emph{only} involve firm $p_{i}$ within $\mw$. By Theorem 4.2 in \citep{abeledo1995paths}, we can start from any matching $u$ to a stable matching by iteratively satisfying blocking pairs in a \emph{gender consistent} order, which means that we can provide a well-defined order to determine which blocking triplet should be satisfied (matched) first within preferences from firm $p_{i}$\footnote{This gender consistent requirement is to satisfy a blocking pair $(p_{j}, a_{k}^{m})$ and those blocking pairs can be ordered before we break their current matches if any, and then match $p_{j}$ and $a_{k}^{m}$ to get a new matching.}. Doing so, firm $p_{i}$ can never get a worse match than $a_{l}^{m}$ since a blocking pair will let firm $p_{i}$ match with a better $\mw$ than $a_{l}^{m}$, or become unmatched as the algorithm proceeds, so the matching will remain unstable. The matching will continue, which is a contradiction.

Hence there must exist a firm $p_{j} \neq p_{i}$ such that $p_{j}$ is part of a blocking triplet in $u$ when firm $p_{i}$ is matched with $\mw$ $a_{l}^{m}$ under the matching $u$. In particular, based on the Theorem 9 (Dubins-Freedman Theorem), firm $p_{j}$ must submit its TS preference.

Let $L_{j,k,k'}^{m}(T)$ be the number of times firm $p_{j}$ matched with $\mw$ $a_{k'}^{m}$ when the triplet $(p_{j}, a_{k}^{m}, a_{k'}^{m})$ is blocking the matching provided by the centralized platform. Then by the definition 
\begin{equation}
\label{eq: TS exceeds times}
    \begin{aligned}
        \sum_{u^{m} \in B_{j,k,k'}^{m}}\text{\textUpsilon}_{u^{m}}(T) = L_{j,k,k'}^{m}(T)
    \end{aligned}
\end{equation}
By the definition of a blocking triplet, we know that if $p_{j}$ is matched with $\mw$ $a_{k'}^{m}$ when the blocking triplet $(p_{j}, a_{k}^{m}, a_{k'}^{m})$ is blocking, the TS sample must have a higher mean reward for $a_{k'}^{m}$ than $a_{k}^{m}$. In other words, we need to bound the expected number of times that the TS mean reward for $\mw$ $a_{k'}^{m}$ is greater than $a_{k}^{m}$.
From \citep{komiyama2015optimal}, we know that the number of times that $(p_{j}, a_{k}^{m}, a_{k'}^{m})$ forms a blocking pair in Thompson sampling, is upper bounded by 
\begin{equation}
\label{eq: TS suboptimal upper bound}
    \begin{aligned}
        \bE L_{j,k,k'}^{m} \leq C^{m}_{i,j,k'}(T) +
            \frac{\log(T)}{d(\mu_{j,\overline{u}_{i,\min}^{m}}, \mu_{j, k'})}
    \end{aligned}
\end{equation}
where $\overline{u}_{i,\min}^{m} = \underset{a_{k}^{m} \in \overline{u}_{j}^{m}}{\argminE}\hspace{1mm} \mu_{i,k}^{m}$ and $C^{m}_{i,j,k'} = \cO((\log(T))^{-1/3})$. The $d(x, y) = x \log(x/y) + (1-x)\log((1-x)/(1-y))$ is the KL divergence between two Bernoulli distributions with expectation $x$ and $y$.

The expected number of times \textUpsilon$_{i,l}^{m}(T)$ a firm $p_{i}$ matched with a $\mw$ such that the mean reward of $a_{l}^{m}$ for firm $p_{i}$ is greater than $p_{i}$'s optimal match $\overline{u}_{i}^{m}$, which is equivalent to the expected number of times viat the achievable sub-matching set \textUpsilon$_{u^{m}}(T)$ where $u^{m} \in \cH_{i,l}^{m}$. So the result then follows from the identity
\begin{equation}
    \bE[\text{\textUpsilon}_{i,l}^{m}(T)] = \sum_{u^{m} \in \cH_{i,l}^{m}} \bE \text{\textUpsilon}_{u^{m}}(T)
\end{equation}
Given a set $\cH^{m}_{i,l}$ of matchings, we say a set $S^{m}$ of triplets $(p_{j}, a_{k}^{m}, a_{k'}^{m})$ is a \emph{cover} of $\cH^{m}_{i,l}$ if
\begin{equation}
    \begin{aligned}
        \wbcup_{(p_{j}, a_{k}^{m}, a_{k'}^{m}) \in S^{m}} B_{j,k,k'}^{m} \supseteq H^{m}_{i,l}
    \end{aligned}
\end{equation}
Let $\cC(H^{m}_{i,l})$ denote the set of covers of $H^{m}_{i,l}$.
Then 
\begin{equation}
    \begin{aligned}
        \bE[\text{\textUpsilon}_{i,l}^{m}(T)] 
        &=
        \bE \sum_{u^{m} \in \cH_{i,l}^{m}}  \text{\textUpsilon}_{u^{m}}(T)\\
        &\leq \bE \min_{S^{m} \in \cC(\cH_{i,l}^{m})} \sum_{(p_{j}, a_{k}^{m}, a_{k'}^{m}) \in S^{m}} \text{\textUpsilon}_{u^{m}}(T)\\
        &= \min_{S^{m} \in \cC(\cH_{i,l}^{m})} \bE \sum_{(p_{j}, a_{k}^{m}, a_{k'}^{m}) \in S^{m}}
        \text{\textUpsilon}_{u^{m}}(T) \\
        &= \min_{S^{m} \in \cC(\cH_{i,l}^{m})} \sum_{(p_{j}, a_{k}^{m}, a_{k'}^{m}) \in S^{m}} \bE  L_{j,k,k'}^{m}(T)\\
        &\leq
        \min_{S^{m} \in \cC(\cH_{i,l}^{m})} \sum_{(p_{j}, a_{k}^{m}, a_{k'}^{m}) \in S^{m}}
        \Bigg(C^{m}_{i,j,k'}(T) +
            \frac{\log(T)}{d(\mu_{j,k}, \mu_{j, k'})}
        \Bigg)\\
        &\leq
        \min_{S^{m} \in \cC(\cH_{i,l}^{m})} \sum_{(p_{j}, a_{k}^{m}, a_{k'}^{m}) \in S^{m}}
        \Bigg(C^{m}_{i,j,k'}(T) +
            \frac{\log(T)}{d(\mu_{j,\overline{u}_{i,\min}^{m}}, \mu_{j, k'})}
        \Bigg)
    \end{aligned}
\end{equation}
where the first inequality is from the property of cover and we select the minimum cover $S^{m}$ from $\cC(\cH_{i,l}^{m})$. And summation in the third line is equivalent to $\sum_{u^{m} \in B^{m}_{j,k,k'}}$. Based on Eq.~\eqref{eq: TS exceeds times}, the third equality is obvious. From \citep{komiyama2015optimal}, we know the expected number of times of matching with the sub-optimal $\mw$ is upper bounded by Eq.~\eqref{eq: TS suboptimal upper bound}. 
\end{proof}

\section{Firm DA Algorithm with type and without type consideration}
\label{supp-sec: algos}
In this section, we present the DA algorithm with type consideration and without type consideration.

\begin{algorithm}[t]
\SetAlgoLined
	\DontPrintSemicolon
	\SetAlgoLined
	\SetKwInOut{Input}{Input}
    \SetKwInOut{Output}{Output}
    \SetKwInOut{Initialize}{Initialize}
	\Input{Type. firms set $\cN$, workers set $\cK_{m}, \forall m \in [M]$; firms to workers' preferences $\V{r}_{i}^{m}, \forall i \in [N], \forall m \in [M]$, workers to firms' preferences $\V{\pi}^{m}, \forall m \in [M]$; firms' type-specific quota $q_{i}^{m}, \forall i \in [N], \forall m \in [M]$, firms' total quota $Q_{i}, \forall i \in [N]$.}
    \Initialize{Empty set $\mathcal{S} = \{\}$, empty sets $S^{m} = \emptyset, \forall m \in [M]$.}
        \For{$m = 1,..., M$}{ 
        	\While{$\exists$ A firm $p$ who is not fully filled with the quota $q^{m}$ and has not contacted every $\mw$}{
        	    Let $a$ be the highest-ranking worker in firm $p$'s preference, to whom firm $p$ has not yet contacted.\\
        	    Now firm $p$ contacts the worker $a$.\\
        	    \uIf{Worker $a$ is free}{
        	        $(p,a)$ become matched (add $(p, a)$ to $S^{m}$).\\
                }\uElse{
                    Worker $a$ is matched to firm $p'$ (add $(p', a)$ to $S^{m}$).\\
                    \uIf{Worker $a$ prefers firm $p'$ to firm $p$}{
                        firm $p$ filled number minus 1 (remove ($p, a$) from $S^{m}$).\\
                    }\uElse{
                        Worker $a$ prefers firm $p$ to firm $p'$.\\
                        firm $p'$ filled number minus 1 (remove ($p', a$) from $S^{m}$).\\
                        ($p, a$) are paired (add ($p, a$) to $S^{m}$).\\
                     }
                }
            }   
            Update: Add $S^{m}$ to $\mathcal{S}$.\\
        }
    \Output{Matching result $\mathcal{S}$.}
    \caption{Firm-Proposing DA Algorithm with Type Consideration.}
    \label{algo:da-type}
\end{algorithm}

\begin{algorithm}[H]
\SetAlgoLined
	\DontPrintSemicolon
	\SetAlgoLined
	\SetKwInOut{Input}{Input}
    \SetKwInOut{Output}{Output}
    \SetKwInOut{Initialize}{Initialize}
	\Input{Worker Types, firms set $\cN$, workers set $\cK_{m}, \forall m \in [M]$; firms to workers' preferences $\V{r}_{i}^{m}, \forall i \in [N], \forall m \in [M]$, workers to firms' preferences $\V{\pi}^{m}, \forall m \in [M]$; firms' type-specific quota $q_{i}^{m}, \forall i \in [N], \forall m \in [M]$, firms' total quota $Q_{i}, \forall i \in [N]$.} 
    \Initialize{Empty set $S$.}
        \While{$\exists$ A firm $p$ who is not fully filled with the quota $ \tilde{Q}$ and has not contacted every worker}{
            Let $a$ be the highest-ranking worker in firm $p$'s preference over all types of workers, to whom firm $p$ has not yet contacted.\\
            Now firm $p$ contacts the worker $a$.\\
            \uIf{Worker $a$ is free}{
                $(p,a)$ become matched (add $(p, a)$ to $S$).\\
            }\uElse{
                Worker $a$ is matched to firm $p'$ (add $(p', a)$ to $S$).\\
                \uIf{Worker $a$ prefers firm $p'$ to firm $p$}{
                    firm $p$ filled number minus 1 (remove ($p, a$) from $S$).\\
                }\uElse{
                    Worker $a$ prefers firm $p$ to firm $p'$.\\
                    firm $p'$ filled number minus 1 (remove ($p', a$) from $S$).\\
                    ($p, a$) are paired (add ($p, a$) to $S$).\\
                 }
            }
        }
    \Output{Matching result $S$.}
	\caption{Firm-Proposing DA Algorithm without Type Consideration \citep{gale1962college}. }
 \label{algo:da-no-type}
\end{algorithm}

\section{Experimental Details}
\label{supp-sec: exps}
In this section, we provide more details about the learned parameters and large market.

\begin{table}[t]
\centering
\caption{True Matching Scores of two types of workers from two firms.}
\begin{tabular}{ccccccc}
\toprule
\multirow{1}{*}{Mean ID}  & \multicolumn{1}{c}{\textbf{Type}} & \multicolumn{1}{c}{\textbf{1}} & \textbf{2}  & \textbf{3} & \textbf{4} & \textbf{5} \\ \hline
\multirow{2}{*}{$\V{\mu}_{1}$}   & 1  & 0.406 & 0.956 & 0.738 & 0.970 & 0.695 \\ \cline{2-7}
 & 2 &  0.932 & 0.241 & 0.040 & 0.657 & 0.289  \\\midrule
\multirow{2}{*}{$\V{\mu}_{2}$}   & 1  & 0.682 & 0.909 & 0.823& 0.204 & 0.218  \\ \cline{2-7}
  & 2 & 0.303 & 0.849 & 0.131 & 0.886 & 0.428 \\\bottomrule
\end{tabular}
\label{table: mean reward of negative regret}
\end{table}

\subsection{Learning Parameters}
\label{sec: learning parameters}
In this section, we present the learning parameters of $(\V{\alpha}, \V{\beta})$ of Example 1.
Besides, we analyze which kind of pattern causes the non-optimal stable matching of Examples 1 and 2.

\begin{figure}
    \centering
    \includegraphics[scale=0.5]{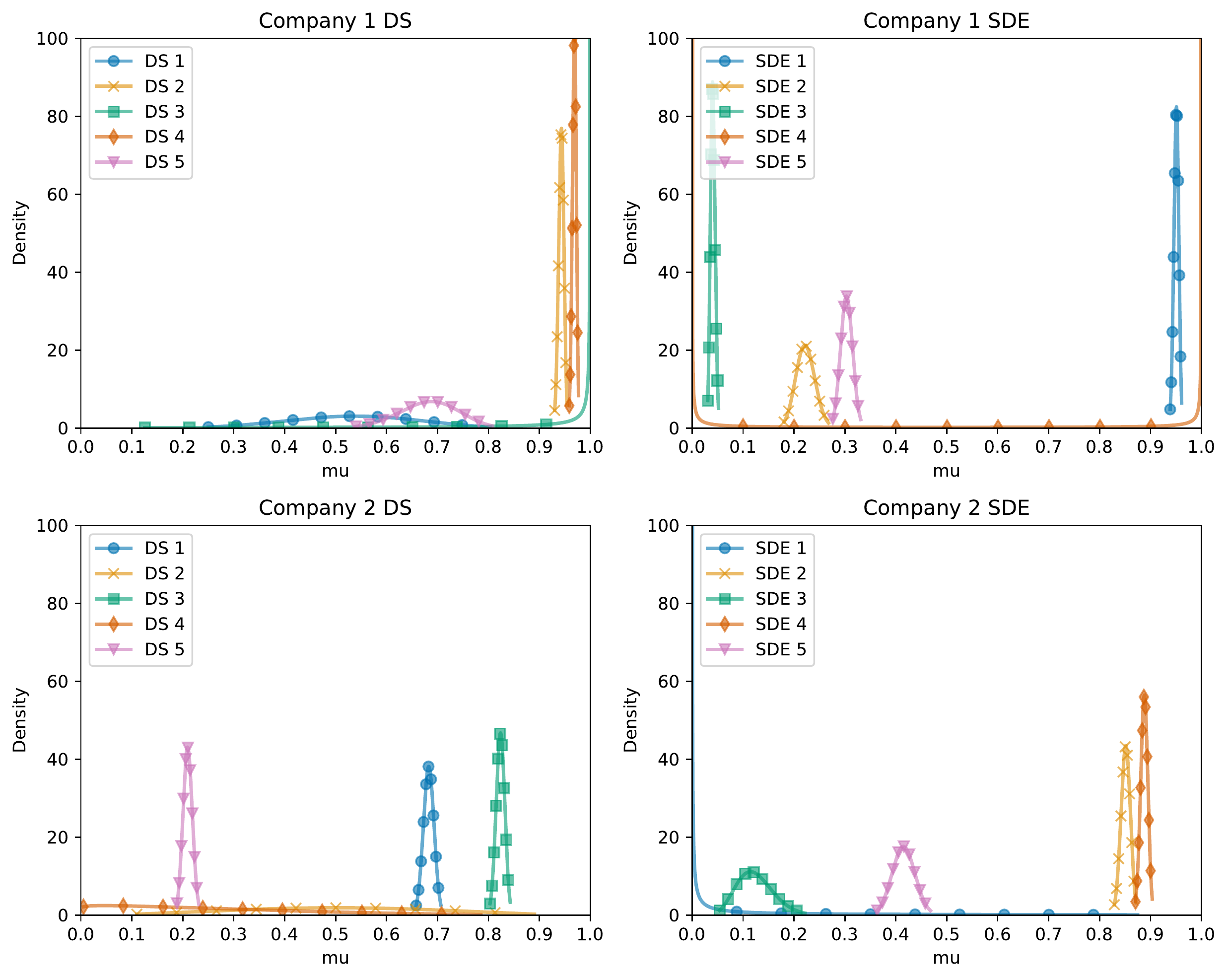}
    \vspace{-0.1in}
    \caption{Posterior distribution of learning parameters for two firms in Example 1.}
    \label{fig:posterior}
\end{figure}

\begin{table}[t]
\centering
\caption{Estimated mean reward and variance of each type of worker in view of two firms. The bold font is to represent the firm's optimal stable matching. $\dag$ represents the difference between the estimated mean and the true mean less than $1\%$. $\ddag$ represents the difference is less than $1.5\%$.}
\begin{tabular}{ccccccc}
\toprule
\multirow{1}{*}{Mean \& Var}  & \multicolumn{1}{c}{\textbf{Type}} & \multicolumn{1}{c}{\textbf{1}} & \textbf{2}  & \textbf{3} & \textbf{4} & \textbf{5} \\ \hline
\multirow{2}{*}{$\hat{\V{\mu}}_{1}$}   
 & 1 (DS) & $0.533_{0.015}$ & $\textbf{0.943}_{0.000}^{\ddag}$ & $0.917_{0.035}$ & $\textbf{0.968}_{0.000}^{\dag}$ & $0.682_{0.003}^{\ddag}$ \\ \cline{2-7}
 & 2 (SDE) & $\textbf{0.950}_{0.000}$ & $0.223_{0.000}$ & $\textbf{0.041}_{0.000}^{\dag}$ & $0.500_{0.208}$ & $\textbf{0.303}_{0.000}^{\ddag}$  \\\midrule
\multirow{2}{*}{$\hat{\V{\mu}}_{2}$}   
 & 1 (DS)  & $\textbf{0.683}_{0.000}^{\dag}$ & $0.500_{0.035}$ & $\textbf{0.823}_{0.000}^{\dag}$ & $0.262_{0.037}$ & $\textbf{0.210}_{0.000}^{\dag}$  \\ \cline{2-7}
  & 2 (SDE) & $0.083_{0.035}$ & $\textbf{0.851}_{0.000}^{\dag}$ & $0.124_{0.001}^{\dag}$ & $\textbf{0.887}_{0.000}^{\dag}$ & $0.415_{0.001}^{\ddag}$ \\\bottomrule
\end{tabular}
\label{table: posterior mean of the example 1}
\end{table}

We show the posterior distribution of $(\V{\alpha}, \V{\beta})$ in Figure \ref{fig:posterior}. 
The first and second row represents the posterior distributions of firm 1 and firm 2 over two types of workers after $T$ rounds interaction. 
The first and second columns in Figure \ref{fig:posterior} represent two firms' posterior distributions over type I and type II workers. 

We find that the posterior distributions of the workers that firms most frequently match with exhibit a relatively sharp shape, indicating that firms can easily construct uncertainty sets over these workers. However, in some instances, the distributions are relatively flat, indicating a lack of exploration. This can be attributed to two possible reasons: (1) the workers in question are not optimal stable matches for the firms, and are thus abandoned early on in the matching process, such as firm 1's DS 1 and DS 5, or (2) the workers are optimal, but are erroneously ranked by the firms and subsequently blocked, such as firm 2's SDE 3. To further illustrate this, we present the posterior mean and variance in Table \ref{table: posterior mean of the example 1}. The optimal stable matches for each firm are represented in bold, and the variance of the distributions is denoted by small font. Additionally, we use the dagger symbol to indicate when the difference between the posterior mean reward and true matching score is less than $1\%$ and $1.5\%$.

\textbf{Pattern Analysis.} We find that firm 1's type I matching in Figure \ref{fig: neg regret}, achieves a negative regret due to the high-frequency matching pattern of $\V{u}_{1} = \{[D_{4}, D_{2}, D_{5}], [S_{1}, S_{5}]\}$, and $\V{u}_{2} = \{ [D_{3}, D_{1}], [S_{4}, S_{2}, S_{3}]\}$. 
That means firm 1 and firm 2 have a correct (stable) matching in the first match $\tilde{\V{u}}_{1} = \{[D_{4}, D_{2}], [S_{1}, S_{5}]\}, \tilde{\V{u}}_{2} = \{ [D_{3}, D_{1}], [S_{4}, S_{2} ]\}$.
In the second match, they both need to compare worker $D_{5}$ and worker $S_{3}$, because all other workers are matched with firms or have been proposed in the first match. 
In Table \ref{table: mean reward of negative regret}, we find that two workers' true mean rewards for firm 1 are $\mu_{1,5}^{1} = 0.695, \mu_{1,3}^{2} = 0.040$ and two workers' estimated rewards for firm 1 are  $\widehat{\mu}_{1,5}^{1} = 0.682, \widehat{\mu}_{1,3}^{2} = 0.041$. These two workers are pretty different and can be easily detected. 
So firm 1 has a high chance of ranking them correctly. 
However, two workers' true rewards for firm 2 are $\mu_{2,5}^{1} = 0.218, \mu_{2,3}^{2} = 0.131$, and two workers' estimated rewards for firm 2 are $\widehat{\mu}_{1,5}^{1} = 0.210, \widehat{\mu}_{1,3}^{2} = 0.124$. These workers are close to each other, where these two posteriors' distributions overlap a lot and can be checked in Figure \ref{fig:posterior}.
So firm 2 has a non-negligible probability to incorrectly rank $S_{3}$ ahead of $D_{5}$. 
Therefore, based on the true preference, firm 2 could match with $S_{3}$ and firm 1 matches with $D_{5}$ with a non-negligible probability rather than the optimal stable matching $(p_{1}, S_{3})$ and $(p_{2}, D_{5})$ by $D_{5}$ preferring firm 2. 

The above pattern links to Section \ref{sec: insufficient exploration-ucb-vs-ts}, incapable exploration, and Section \ref{sec: strategy}, incentive compatibility. Due to the insufficient exploration of $S_{3}$ and $D_{5}$, firm 2 may rank them incorrectly to get a match with $S_{3}$ rather than optimal $D_{3}$ and the regret gap is $\mu_{2,3}^{1} - \mu_{2, 3}^{2} = 0.823 - 0.131 = 0.692$, which is a positive instantaneous regret. Due to the incorrect ranking from firm 2, firm 1 gets a final match with $D_{5}$ rather than optimal $S_{3}$, and suffers a regret gap $\mu_{1,3}^{2} - \mu_{1,5}^{1} = 0.040 - 0.695 = - 0.655$, which is a negative instantaneous regret. Thus firm 1 benefits from firm 2's incorrect ranking and can achieve a total negative regret, as shown in Figure \ref{fig: neg regret}.

\paragraph{Findings from Example 2.}
In our analysis of the non-optimal stable matching in Example 2, we observed that both firms incurred positive total regret, shown in Figure \ref{fig: first match example 2}. We find that the quota setting resulted in all workers of type II being assigned to firms in the first match. As a result, in the second match, the ranking submitted by firm 1 to the centralized platform did not affect firm 2's matching result for type II workers. This can be thought of as an analogy where firms are schools and workers are students. In the second stage of the admission process, school 2 would not participate in the competition for type II students, and its matching outcome would not be affected by the strategic behavior of other schools in the second stage, but rather by the strategic behavior of other schools in the first stage.

\subsection{Large markets}
\label{sec: large market}
In this part, we provide two large market examples to demonstrate the robustness of our algorithm. All preferences are randomly generated and all results are over 50 trials to take the average.

\paragraph{Example 3.} We consider a large market composed of many firms ($N=100$) and many workers ($K_{1} = K_{2} = 300$). Besides, we have $Q_{1} = Q_{2} = 3, q_{1}^{1} = q_{2}^{1} = q_{2}^{1} = q_{2}^{2} = 1$.

\paragraph{Example 4.} We also consider a large market consisting of many workers, and each firm has a large, specified quota and an unspecified type quota. In this setting, $N=10, M = 2, K_{1} = K_{2} = 500, Q_{1} = Q_{2} = 30, q_{1}^{1} = q_{2}^{1} = q_{2}^{1} = q_{2}^{2} = 10$.

\paragraph{Results.} In Figure \ref{fig: large example 3}, we randomly select 10 out of 100 to present firms' total regret, and all those firms suffer sublinear regret. In Figure \ref{fig: large example 4}, we also show all 10 firms' total regret. Comparing Examples 3 and 4, we find that firms' regret in Example 3 is less than firms' regret from Example 4 because in Example 4, each firm has more quotas (30 versus 3), which demonstrates our findings from Theorem \ref{thm: regret upper bound}. 
In addition, we find there is a sudden exchange in Figure \ref{fig: large example 3} nearby time $t=1500$. We speculate this phenomenon is due to the small gap between different workers and the shifting of the explored workers.

\begin{figure}%
    \centering
    \subfigure{%
    \label{fig: large example 3}%
    \includegraphics[width=0.45\textwidth, height=0.3\textwidth]{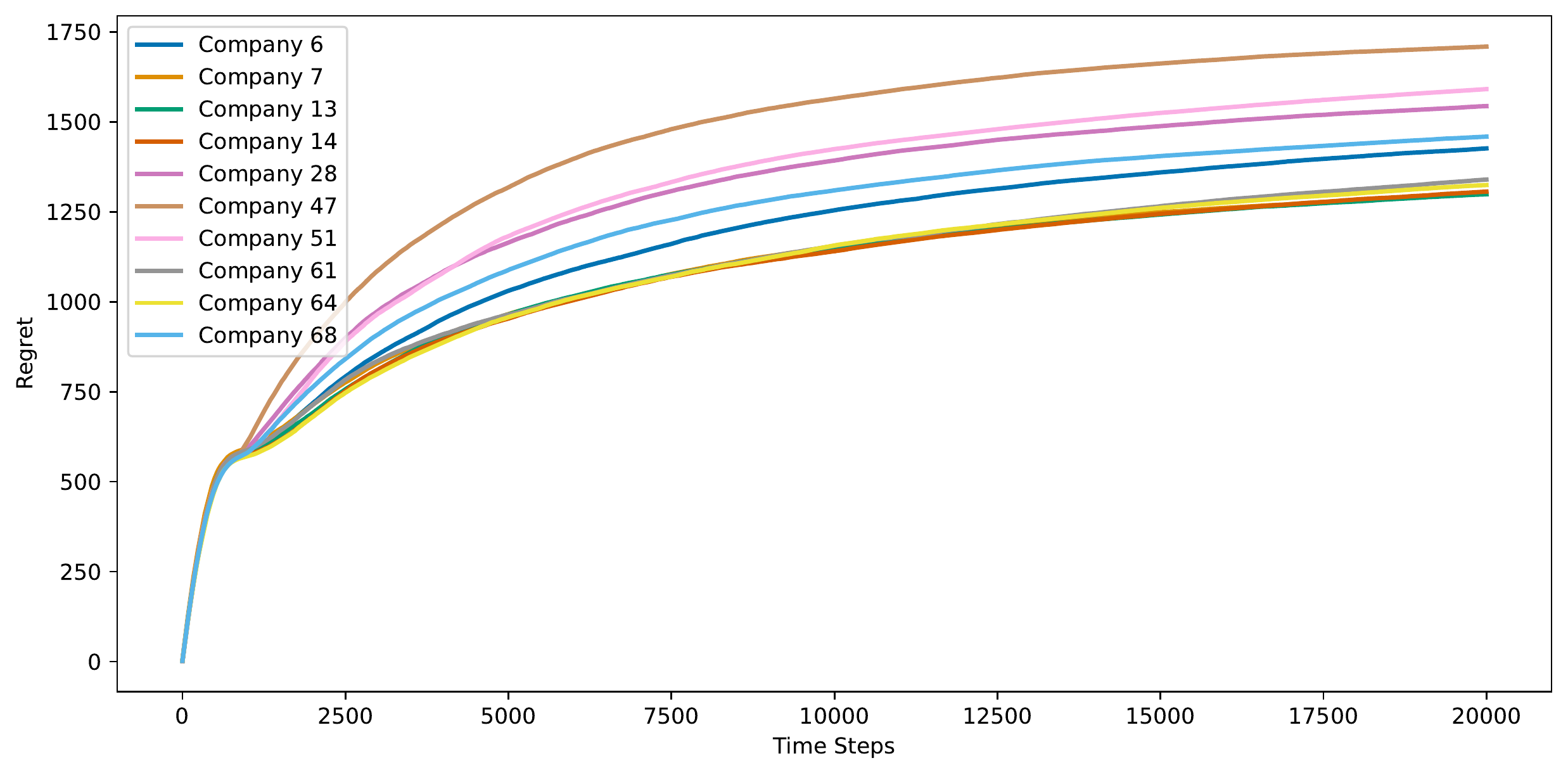}}%
    \subfigure{%
    \label{fig: large example 4}%
    \includegraphics[width=0.45\textwidth, height=0.3\textwidth]{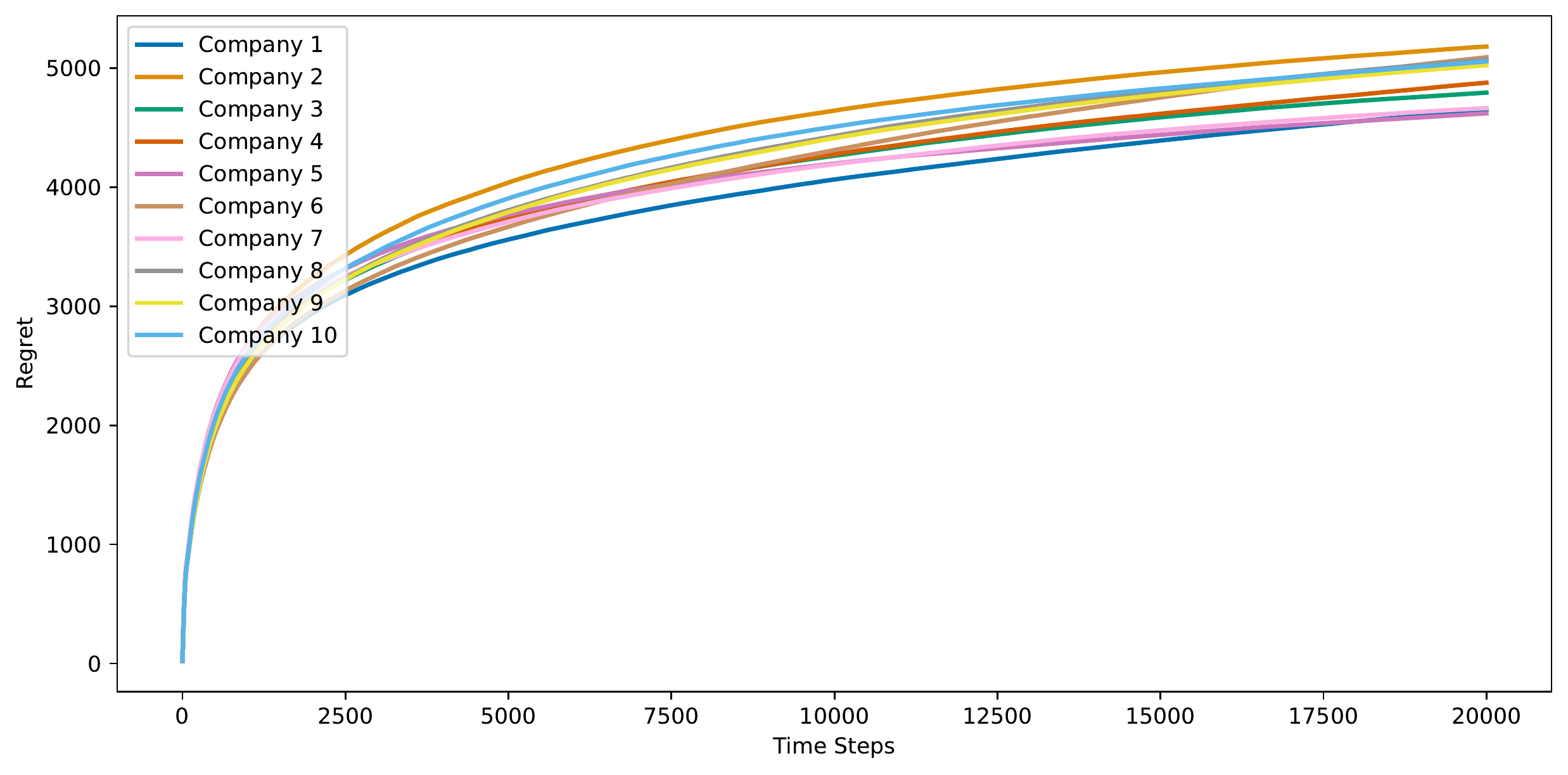}}%
    \caption{Left: 10 out of 100 randomly selected firms' total regret in Examples 3. Right: all firms' total regret in Example 4.}
    \label{fig: large market}
\end{figure}

\end{document}